\documentclass[11pt]{amsart}
\usepackage[utf8]{inputenc}
\usepackage{kantlipsum}
\usepackage[pagebackref=true, colorlinks=true,linkcolor=blue, citecolor={red}]{hyperref}
\usepackage{lmodern}\usepackage[T1]{fontenc}
\usepackage{amsmath,amssymb,amsfonts,euscript,graphicx,hyperref,indentfirst}
\usepackage{enumerate}
\usepackage[all]{xy}
\usepackage{relsize,exscale}
\usepackage{textcomp}
\calclayout

\newtheorem{theorem}{{\bf Theorem}}[section]

\newtheorem{preexample}[theorem]{{\bf Example}}

\newenvironment{example}{\begin{preexample}\rm{\hspace{-0.5
              em}{\bf}}}{\end{preexample}}

\newtheorem{prelem}[theorem]{{\bf Lemma}}

\newenvironment{lemma}{\begin{prelem}{\hspace{-0.5
              em}{\bf}}}{\end{prelem}}

\newtheorem{preprop}[theorem]{{\bf Proposition}}

\newenvironment{proposition}{\begin{preprop}{\hspace{-0.5
              em}{\bf}}}{\end{preprop}}

\newtheorem{prequ}[theorem]{{\bf Question}}

\newenvironment{question}{\begin{prequ}\rm{\hspace{-0.5
              em}{\bf}}}{\end{prequ}}

\newtheorem{predef}[theorem]{{\bf Definition}}

\newenvironment{definition}{\begin{predef}\rm{\hspace{-0.5
              em}{\bf}}}{\end{predef}}

\newtheorem{precor}[theorem]{{\bf Corollary}}

\newenvironment{corollary}{\begin{precor}{\hspace{-0.5
              em}{\bf}}}{\end{precor}}

\newtheorem{preconj}[theorem]{{\bf Conjecture}}

\newtheorem{preex}[theorem]{{\bf Exercise}}

\newtheorem{prerem}[theorem]{{\bf Remark}}

\newenvironment{remark}{\begin{prerem}\rm{\hspace{-0.5
              em}{\bf}}}{\end{prerem}}

\makeatletter
\def\l@subsection{\@tocline{2}{0pt}{2.5pc}{5pc}{}} 
\makeatother

\numberwithin{equation}{section}

\title{On functions computed on trees}
\author{Roozbeh Farhoodi\,$^*$, Khashayar Filom\,$^*$, Ilenna Simone Jones\\   Konrad Paul Kording}

\address{Roozbeh Farhoodi, Department of Bioengineering and Department of Neuroscience at University of Pennsylvania
404 Richards Building, 3700 Hamilton Walk, Philadelphia, PA 19104, Philadelphia, PA 19104}
\email{roozbeh@seas.upenn.edu}

\address{Khashayar Filom, Department of Mathematics,
  Northwestern University;
 2033 Sheridan Road, Evanston, IL 60208,
USA}
\email{khashayarfilom2014@u.northwestern.edu}

\address{Ilenna Simone Jones, Department of Bioengineering and Department of Neuroscience at University of Pennsylvania
404 Richards Building, 3700 Hamilton Walk, Philadelphia, PA 19104, Philadelphia, PA 19104}
\email{ilennaj@pennmedicine.upenn.edu}

\address{Konrad Paul Kording, Department of Bioengineering and Department of Neuroscience at University of Pennsylvania
404 Richards Building, 3700 Hamilton Walk, Philadelphia, PA 19104, Philadelphia, PA 19104}
\email{kording@upenn.edu}

\date{}
\begin{document}
 \maketitle

\begin{abstract}
Any function can be constructed using a hierarchy of simpler functions through compositions. Such a hierarchy can be characterized by a binary rooted tree. Each node of this tree is associated with a function which takes as inputs two numbers from its children and produces one output. Since thinking about functions in terms of computation graphs is getting popular we may want to know which functions can be implemented on a given tree. Here, we describe a set of necessary constraints in the form of a system of non-linear partial differential equations that must be satisfied. Moreover, we prove that these conditions are sufficient in both contexts of analytic and bit-valued functions. In the latter case, we explicitly enumerate discrete functions and observe that there are relatively few. Our point of view allows us to compare different neural network architectures in regard to their function spaces. Our work connects the structure of computation graphs with the functions they can implement and has potential applications to neuroscience and computer science. 
\end{abstract}
\tableofcontents

\section{Introduction}
A complicated function can be constructed by a hierarchy of simpler functions.  For instance, when a microprocessor calculates the value of a function for a given set of inputs, it computes the function through composing simpler implemented functions, e.g. logic gates \cite{haastad1991power, goldmann1992majority}. Another example is that of an addition function for any number of inputs which can be obtained by composing simpler addition functions of two inputs. Even when the set of simple functions is small, like in the case of working only with logic gates,  the set of functions  that  can be  built may be exponentially large.  We know from computation theory that all computable functions can be constructed in this manner \cite{sipser2006introduction, arora2009computational}. Therefore, one approach to understand the set of computable functions is to investigate their potential representations as hierarchical compositions of simpler functions. 

Here we study the set of functions of multiple variables that can be computed by a hierarchy of functions that each accepts two inputs. Such compositions can be characterized by binary rooted trees (in the following we will refer to them as binary trees) that determines the hierarchical order in which the functions of two variables are applied. Associated with any binary tree is a (continuous or discrete) \textit{tree function space (TFS)} consisting of all functions that can be obtained as a composition (\textit{superposition}) based on the hierarchy that the tree provides. In Theorem \ref{main} we exhibit a set of necessary and sufficient conditions for analytic functions of $n$ variables to have a representation  via a given tree. We show that this amounts to describing the corresponding TFS as the solution set to a group of non-linear partial differential equations  (PDEs).  We also study the same representability problem in the context of discrete functions. 

\textbf{Related mathematical background.} Representing multivariate continuous functions in terms of functions of fewer variables has a rich background that roots back to the $13^{\rm{th}}$ problem on David Hilbert's famous list  of mathematical problems for the $20^{\rm{th}}$ century \cite{hilbert1902mathematical}. Hilbert's original conjecture was about describing solutions of $7^{\rm{th}}$ degree equations in terms of functions of two variables. The problem has many variants based on the category of functions -- e.g. algebraic, analytic, smooth or continuous -- in which the question is posed. See \cite[chap. 1]{MR0237729} or the survey article \cite{vitushkin2004hilbert} for a historical account. Later in the 1950s, the Soviet mathematicians Andrey Kolmogorov and Vladimir Arnold did a thorough study of this problem in the context of continuous functions that culminated in the seminal \textit{Kolmogorov-Arnold Representation Theorem} (\cite{MR0111809}) which asserts that every continuous function $F(x_1,\dots,x_n)$ can be described as 
\begin{equation}\label{Kolmogorov-Arnold}
F(x_1,\dots,x_n)=\sum_{i=1}^{2n+1}f_i\left(\sum_{j=1}^n\phi_{i,j}(x_j)\right)
\end{equation}
for suitable continuous single variable functions $f_i$, $\phi_{i,j}$.\footnote{There are more refined versions of this theorem with more restrictions on the single variable functions that appear in the representation  \cite[chap. 11]{MR0213785}.} So in a sense, addition is the only real multivariate function. The idea of applying Kolmogorov-like results to studying networks is not new. Based on the mathematical works of Anatoli Vitu\v{s}kin (see below), the article \cite{girosi1989representation} argues that to quantify the complexity of a function, its number of variables (not a suitable indication of the complexity due to Kolmogorov-Arnold Theorem) must be combined with the degree of smoothness due to the fact that there are highly regular functions that cannot be represented by continuously differentiable functions of a smaller number of variables \cite{MR0062212}. The paper then concludes that it is not possible to obtain an exact representation usable in the context of network theory because of this emergence of non-differentiable functions.  Nevertheless, the article \cite{kuurkova1991kolmogorov} argues that there is an approximation result of this type.\\
\indent Although pertinent to our discussion, the reader should be aware that the representations of multivariate functions studied in this article are different in following ways:
\begin{itemize}
\item Motivated by both the structures of computation graphs and the model of neurons as binary trees, we desire multivariate functions that could be obtained via composition of functions of two variables instead of single variable ones.
\item Unlike the summation above, we work with a \textit{single} superposition of functions. In the presence of differentiability, this enables us to use the full power of the chain rule. In the case of ternary functions for instance, a typical question (to be addressed in \S\ref{three-variable-example}) would be whether $F(x,y,z)$ can be written as $g\left(f(x,y),z\right)$. In fact, if one allows sum of superpositions, a result of Arnold (which could be found in his collected works \cite{arnold2009representation}) states that every continuous $F(x,y,z)$ can be written as a sum of nine superpositions of the form $g\left(f(x,y),z\right)$. But we look for a single superposition, not a sum of them.
\item We mostly work in the analytic context; see \S\ref{smooth setting} for difficulties that may arise if one works with smooth functions. It must be mentioned that assuming that the continuous $F(x_1,\dots,x_n)$  has certain regularity (e.g. smooth or analytic) does not guarantee that in a representation such as \eqref{Kolmogorov-Arnold} the functions can be arranged to be of the same smoothness class \cite{MR0165138}. In fact, it is known that there are always $C^k$ functions\footnote{A function is called (of class) $C^k$ if it is differentiable of order $k$ and its $k^{\rm{th}}$ order partial derivatives are continuous. A ($C^k$) $C^1$ function is said to be (resp. $k$ times) continuously differentiable. A function which is infinitely many times differentiable is called \textit{smooth} or (of class) $C^\infty$. The smaller class of (real) \textit{analytic} functions that are locally given by convergent power series is denoted by $C^\omega$. We refer the reader to \cite{MR1886084} for the standard material from elementary mathematical analysis.}  of three variables which cannot be represented as sums of superpositions of the form $g\left(f(x,y),z\right)$ with $f$ and $g$ being $C^k$ as well \cite{MR0062212}. Because of the constraints that we put on $F$ in our main result, Theorem \ref{main}, it turns out that the functions of two variables that appeared in tree representations are analytic as well, or even polynomial if $F$ is polynomial; see Proposition \ref{polynomial-global}.
\item Applying the chain rule to superpositions of analytic (or just $C^2$) bivariate functions results in the PDE constraints \eqref{temp2}.
 It must be mentioned that the fact that the partial derivatives of functions appearing in any superposition of differentiable functions must be related to each other is by no means new. Hilbert himself has employed this point of view to construct analytic functions of three variables that are not (single or sum of) superpositions of analytic functions of two variables \cite[p. 28]{arnold2009representation1}. Ostrowski for instance, has used this idea to exhibit an analytic bivariate function that cannot be represented as a superposition of single variable smooth functions and multivariate algebraic functions due to the fact that it is not a solution to any non-trivial algebraic PDE \cite{ostrowski1920dirichletsche}, \cite[p. 14]{vitushkin2004hilbert}. But, to the best of our knowledge, a systematic characterization of superpositions of analytic bivariate functions as outlined in Theorem \ref{main} (or its discrete version in Theorem \ref{main binary}) and utilizing that for studying tree functions and neural networks has not appeared in the literature before.              
\end{itemize}

 \textbf{Neuroscience motivation.} Over a century ago, the father of modern neuroscience, Santiago Ram\'{o}n y Cajal, drew the distinctive shapes of neurons \cite{y1995histology}. Neurons receive their inputs on their dendrites which both exhibit non-linear functions and have a tree structure. The trees, called morphologies, are central to neuron simulations \cite{hines1997neuron, reiche1999genesis}. Neuronal morphologies are not just the distinctive shapes of neuron but also pertain to their functions. One approximate way of thinking about neural function is that neurons receive inputs and by passing from the dendritic periphery towards the root, the soma, implement computation which gives a neuron its input-output function. In that view, the question of what a neuron with a given dendritic tree and inputs may compute boils down to the question of characterizing its TFS.
 
 The paper is organized as follows. \S\ref{outline} is devoted to a detailed outline the paper. We also state the main results after a non-technical motivation. In \S\ref{discussion} we discuss the relevant literature from both computer science and neuroscience sides of the theory. Several possible extensions and open problems are stated as well. The central concept of this paper, tree functions, is formally defined in \S\ref{tree functions}. Sections \ref{continuous setting} and \ref{discrete setting} treat tree functions in analytic and bit-valued contexts respectively.  Finally, in \S\ref{neural} we apply these ideas to study neural networks via tree functions.
 
\section{Outline and overview of main results}\label{outline}
The order of appearance of  functions in a superposition can be represented by a tree whose leaves, nodes (branch points) and root represent inputs, functions occurring in the superposition and the output respectively. Here we assume that the tree $T$ and the set of functions that could be applied at each node are given and each leaf is labeled by a variable. We can now define the space of functions generated through superposition, i.e. the corresponding tree function space (TFS) (see Definition \ref{definition-functions}). The most tangible case of a TFS is when all of the inputs are real numbers and the functions assigned to the nodes are bivariate real-valued functions. Nonetheless, our definition in \S\ref{tree functions} covers other cases: an arbitrary tree and sets of functions associated with its nodes result in the set of functions represented by superpositions. One example is when the functions at the nodes are bit-valued functions. Another example is when the inputs are time-dependent  and the functions at nodes are operators. The latter case is important since it contains the function that a neural morphology would implement when we only allow soma-directed influences and ignore back-propagating action potentials \cite{stuart1997action}. 

The smallest non-trivial tree representing a superposition is the one with three leaves illustrated in Figure \ref{fig:three-four}(b). Denoting the inputs by $x$, $y$ and $z$, an element $F$ of the corresponding TFS is the superposition below of a function of two variables $f$ and $g$: 
\begin{equation}\label{result-three-var}
\text{output} =F(x,y,z)= g(f(x,y),z).
\end{equation}
For $f$ and $g$ multiplication or addition, we end up with two basic examples $F(x,y,z) = xyz$ and $F(x,y,z) =x+y+z$. By changing $f$ and $g$ one can construct other examples and hence the question of which functions could be answered in this manner.

To find a necessary condition for a function of three variables $F(x,y,z)$ to have a representation such as \eqref{result-three-var}, we assume differentiability and take the derivative. A straightforward  application of the chain rule to \eqref{result-three-var} shows that  $F$ must satisfy
\begin{equation}\label{condition}
\frac{\partial^2 F}{\partial x\partial z}.\frac{\partial F}{\partial y} = \frac{\partial^2 F}{\partial y\partial z}.\frac{\partial F}{\partial x}.
\end{equation}
A detailed treatment may be found in the discussion from the beginning of \S\ref{three-variable-example}. This partial differential equation for $F$ puts a constraint on functions in the TFS and hence rules out certain ternary functions such as
\begin{equation}\label{non-example}
F(x,y,z) = xyz+x+y+z.    
\end{equation} 

While \eqref{condition} is only a necessary condition, we prove that it is also sufficient in the case of analytic  ($C^\omega$) functions:
\begin{proposition}\label{presentation}
Let $F=F(x,y,z)$ be an analytic function defined on an open neighborhood of the origin $\mathbf{0}\in\Bbb{R}^3$ that satisfies the identity in \eqref{condition}. Then there exist analytic functions $f=f(x,y)$ and $g=g(u,z)$ for which 
$F(x,y,z)=g\left(f(x,y),z\right)$ over some  neighborhood of the origin.  
\end{proposition}
To prove Proposition \ref{presentation}, we look at the Taylor expansion of $F$ with respect to $z$ and argue that each partial derivative has a representation like \eqref{result-three-var}. We then explicitly construct the desired $f$ and $g$ in \eqref{result-three-var} with the help of the Taylor series. Consequently, we arrive at a description of the TFS containing analytic functions of three variables as the set of solutions to a single PDE. 

\begin{figure}
    \centering
    \includegraphics[width=10cm]{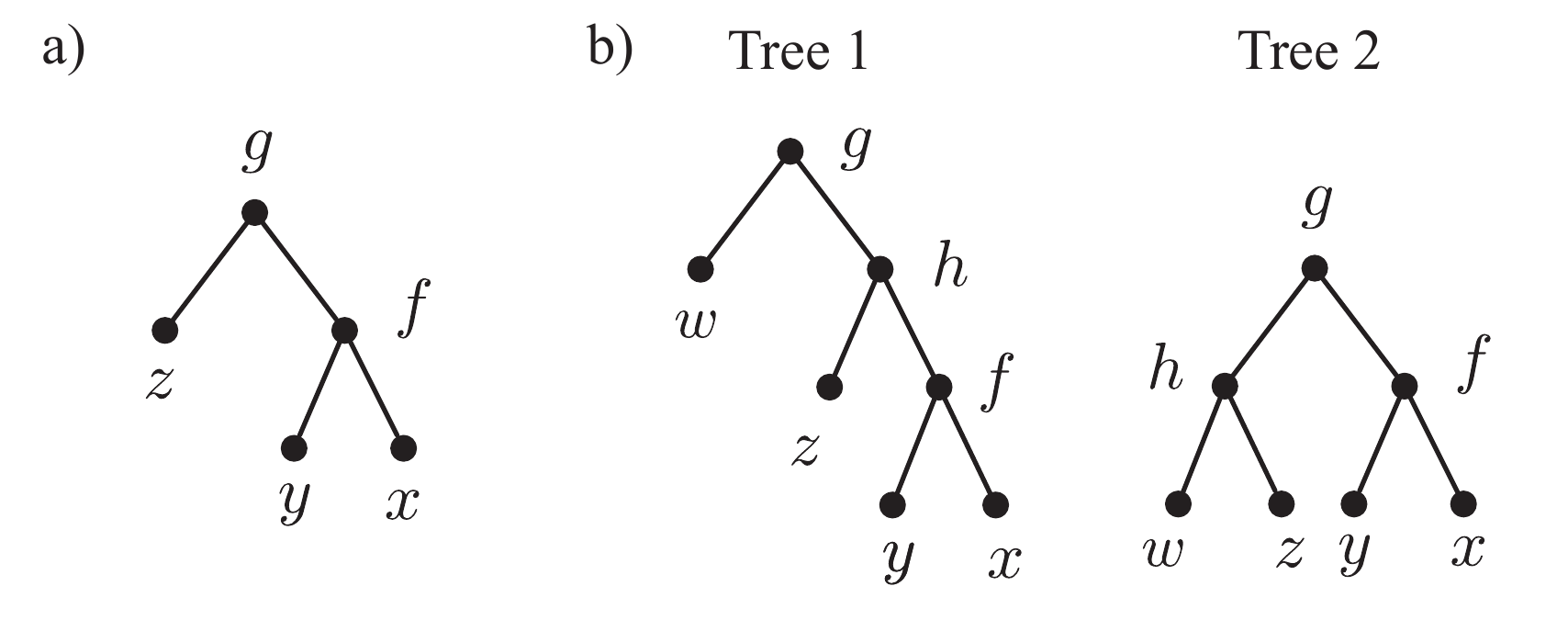}
    \caption{Functions computed on trees. a) The binary tree corresponding to the superposition $g(f(x,y),z)$. The variables are $x,y,z$ and the bivariate functions assigned to the nodes are $f$ and $g$. 
    b) The binary trees with four inputs $x,y,z,w$. Here $f,g,h$ are the functions assigned to the nodes. The corresponding superpositions from left to right are $g(h(f(x,y),z),w)$ and $g(f(x,y),h(z,w))$.}
    \label{fig:three-four}
\end{figure}

Generalizing this setup to a higher number of variables, the following question arises: When can an analytic multivariate function be obtained from composition of functions of two variables? Allowing more than three leaves results in graph-theoretically distinct binary trees.  For example, in the case of functions of four variables, there exist two non-isomorphic  binary trees; Figure \ref{fig:three-four}. The corresponding representations are
\begin{equation*}
F(x,y,z,w)= g(h(f(x,y),z),w)
\end{equation*}
for the first tree and
\begin{equation*}
F(x,y,z,w)= g(f(x,y),h(z,w))
\end{equation*}
for the second one. Thus each (labeled) binary tree $T$ comes with its corresponding space $\mathcal{F}(T)$ of analytic \textit{tree functions} that could be obtained from analytic functions on smaller number of variables via composition according to the hierarchy that $T$ provides; see  Definition \ref{definition-functions}.

Condition \eqref{condition} from the ternary case is the prototype of constraints that general smooth functions from a TFS must satisfy. By fixing  $n-3$ variables of a function of $n$ variables in the TFS under consideration, the resulting function of three variables belongs to the TFS of the tree formed by those three leaves and is hence a solution to a PDE of the form \eqref{condition}. Since this is true for any triple of variables, numerous  necessary conditions must be imposed. In Theorem \ref{main} we prove that  for analytic functions, these conditions are again sufficient. 
\begin{theorem}\label{main}
Let $T$ be a binary tree with $n$ terminals and $F\in \mathcal{F}(T)$. Suppose the terminals of $T$ are labeled by the coordinate functions $x_1,\dots,x_n$ on $\Bbb{R}^n$.
Then for  any three leaves of $T$ corresponding to variables $x_i,x_j,x_l$ of $F$ with the property that there is a  sub-tree of $T$ containing the leaves $x_i,x_j$ while missing the leaf $x_l$ (Figure \ref{fig:predecessor})\footnote{Clearly, there is always a rooted sub-tree that separates one of the leaves $x_i,x_j,x_l$ from the other two: consider the smallest rooted sub-tree (cf. \S\ref{tree functions} or Figure \ref{fig:tree} for the terminology) that has all of them as leaves. Adjacent to the root of this smaller binary tree are its left and right sub-trees. One of them must contain two of $x_i,x_j,x_l$ and the other one has the third leaf.}, $F$ must satisfy  
\begin{equation}\label{temp2}
\frac{\partial^2 F}{\partial x_i\partial x_l}.\frac{\partial F}{\partial x_j} = \frac{\partial^2 F}{\partial x_j\partial x_l}.\frac{\partial F}{\partial x_i}.
\end{equation}
Conversely, an analytic function $F$ defined in a neighborhood of a point $\mathbf{p}\in\Bbb{R}^n$ can be implemented on the tree $T$ provided that for any triple $(x_i,x_j,x_l)$ of its variables with the above property \eqref{temp2} holds and moreover, for any two sibling leaves $x_i,x_{i'}$, either $\frac{\partial F}{\partial x_i}(\mathbf{p})$ or  $\frac{\partial F}{\partial x_{i'}}(\mathbf{p})$ is non-zero.
\end{theorem}

\begin{figure}
    \centering
    \includegraphics[width=7cm]{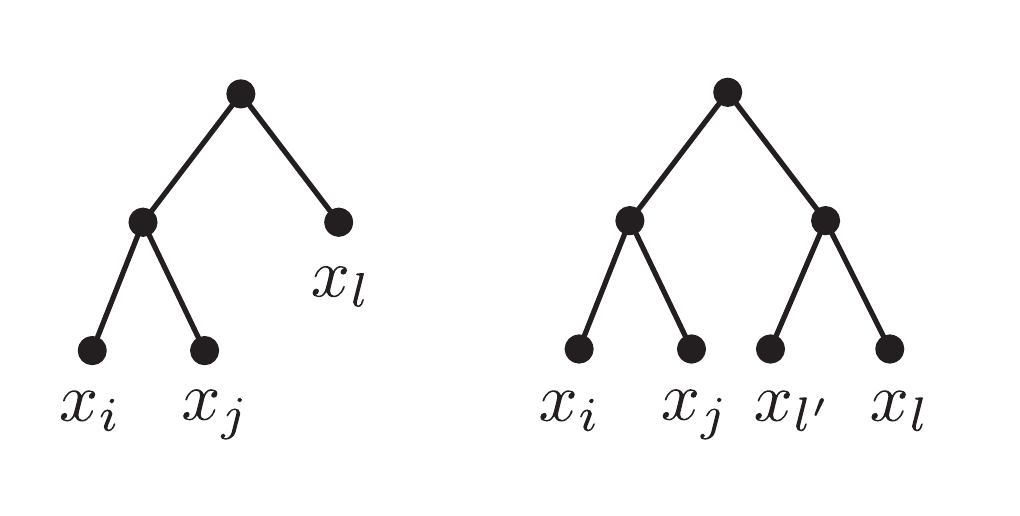}
    \caption{Theorem \ref{main} poses constraints of form \eqref{temp2} on partial derivatives w.r.t. any triple of variables. Thinking of variables as leaves of the tree, in any such a triple there is an \textit{outsider} which is separated from the other two leaves via a rooted sub-tree. For instance, in the trees above $x_l$ and $x_{l'}$ are outsiders of the triples $\{x_i,x_j,x_l\}$ and $\{x_i,x_j,x_{l'}\}$ respectively.}
    \label{fig:predecessor}
\end{figure}

\begin{figure}
    \centering
    \includegraphics[width=6cm]{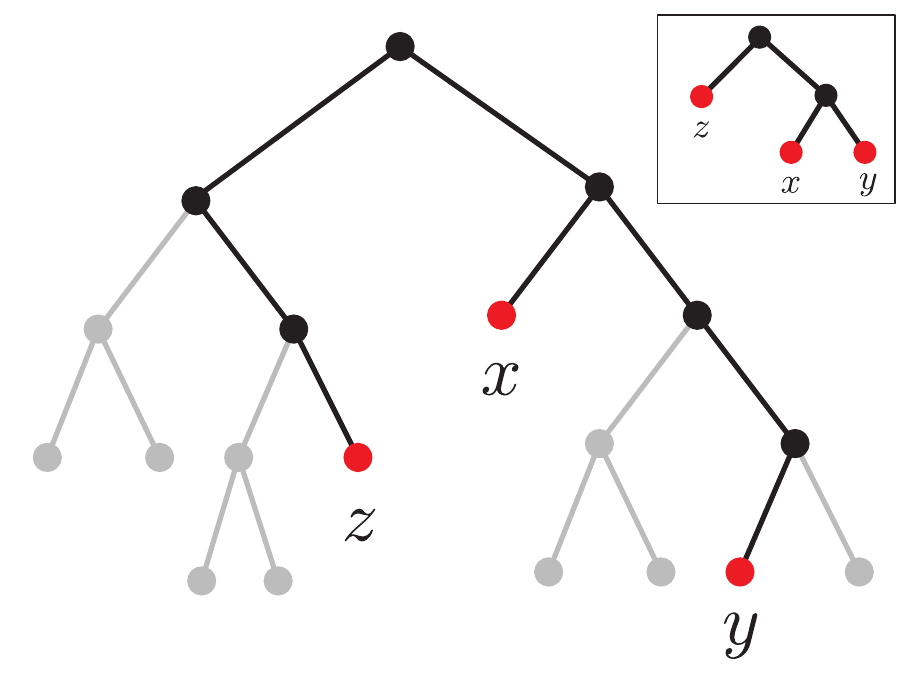}
    \caption{A tree function with $n$ inputs (here $n=10$) is shown. Fixing $n-3$ variables (gray leaves), we get a ternary function (red leaves). PDEs \eqref{temp2} appeared in Theorem \ref{main} are instances of the constraint \eqref{condition} written for such ternary functions.}
    \label{fig:necessary}
\end{figure}

The general argument for the trees with larger number of leaves builds on the proof in the case of ternary functions demonstrated above. The proof occupies \S\ref{the proof} and heavily uses the analyticity assumption. We digress to the setting of smooth ($C^\infty$) function in \S\ref{smooth setting} to show that this assumption cannot be dropped. 

The constraints in the Theorem \ref{main} are algebraically dependent. The number of the constraints imposed by Theorem \ref{main} is ${n \choose 3}$, hence cubic in the number of leaves. In  \S\ref{codimension} we show that ``generically'' the number of constraints could be reduced to ${n-1 \choose 2}$. This leads to  a heuristic\footnote{We use the term ``heuristic'' both because the space of analytic functions that Theorem \ref{main} deals with is infinite-dimensional and so one should be careful about the exact meaning of co-dimension; and moreover, because the number of constraints is not necessarily synonymous to the deficit of the dimension of the subspace of tree functions. This is because each constraint in the form of \eqref{temp2} is a functional identity that could amount to multiple constraints on the coefficients of the Taylor expansion.} result on the co-dimension of the TFS in Proposition \ref{binomial growth} which states that the number of independent functional equations describing a TFS  grows only quadratically with the number of leaves.

The space of  analytic functions is infinite-dimensional and this makes it difficult to rigorously measure how ``small'' a TFS is relative to the ambient space of all analytic functions. However, under certain restrictions, the dimensions of the tree function space or even the space itself are finite. 	Two examples are worthy of  investigation: bit-valued functions  of the form $\{0,1\}^n\rightarrow\{0,1\}$ and polynomials of bounded degree. In the bit-valued setting of \S\ref{discrete-tree} each  node is characterized by a function $\{0,1\}^2\rightarrow\{0,1\}$. We prove that Theorem \ref{main} still holds in the sense of formal differentiation; see Theorem \ref{main binary}.   Moreover, we enumerate the functions in the discrete TFS as $\frac{2\times 6^n+8}{5}$ (Corollary \ref{formula}), a number which is much smaller than the number of all possible bit-valued functions which is $2^{2^n}$. We use this to conclude that the total number of tree functions of $n$ variables obtained from all labeled  binary trees of $n$ leaves is $o\left(2^{2^n}\right)$; see Corollary \ref{sparse}.  In the polynomial setting, each node is a bivariate polynomial. We establish that the Theorem \ref{main} applies and furthermore, holds globally; cf. Proposition \ref{polynomial-global}. In this case, if we consider polynomials of $n$ variables and of degree not greater than $k$, the polynomial TFS would be of an algebraic variety whose dimension does not exceed $n(k^2+1)$; see Proposition \ref{dimension bound}.  Again, observe that this is much smaller than the dimension $\binom{k+n}{n}$ of the ambient polynomial space.

The set of binary functions that can be implemented on a given tree is limited and this set allows the reconstruction  of the underlying tree from its corresponding TFS; see Proposition \ref{distinguish}.  For two labeled trees we define a metric: the proportion of functions that can only be represented by one of the trees (\S\ref{distance}). This can be useful: in the case of two neurons with different morphologies, this simple metric quantifies how similar the sets of functions are that the two neurons can implement.

\begin{figure}
    \centering
    \includegraphics[width=12cm]{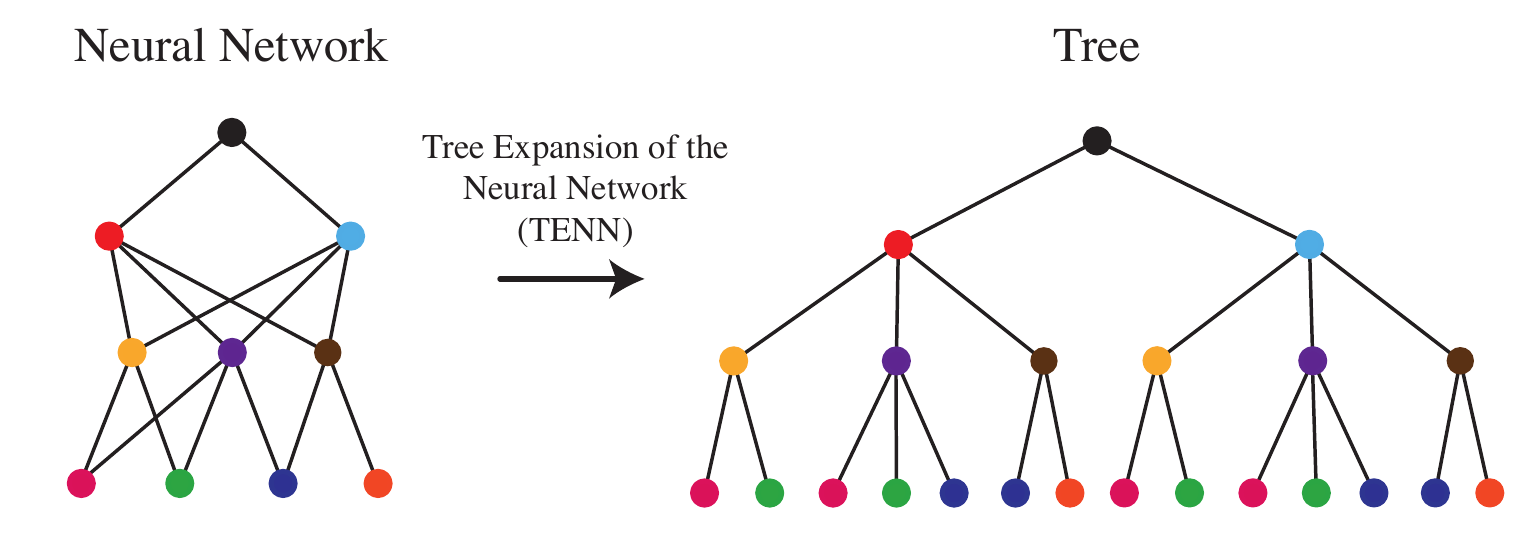}
    \caption{The TENN procedure for constructing a tree out of a multi-layer neural network.}
    \label{fig:tree-expanded-network}
\end{figure}
In a more general manner, the functions defined by neural networks are interesting examples of superpositions. In \S\ref{neural to tree} we discuss a procedure of ``expanding'' a neural network to a tree by forming the corresponding  \textit{\textbf{T}ree \textbf{E}xpansion  of the \textbf{N}eural \textbf{N}etwork} or \textbf{TENN} for short. The idea is to convert the neural network of interest into a tree by duplicating the nodes that are connected to more than one node of the upper layer; see Figure \ref{fig:tree-expanded-network}. The procedure then allows us to revert back to the familiar setting of trees. A crucial point to notice is that, unlike previous sections, trees associated with neural networks are not necessarily binary and furthermore, a variable could appear as the label of more than one leaf. In other words, the functions constituting the superposition may have variables in common, e.g. $F(x,y,z)=g(f(x,y),h(x,z))$.  Seeking similar constraints for describing the TFS of a tree with repeated labels, in \S\ref{labeled; exact} we take a closer look at the preceding superposition to obtain a necessary condition:
\begin{proposition}\label{pde-for-g(f(x,y),h(x,z))}
Assuming that $f, g, h$ are four times differentiable, the superposition $F(x,y,z) = g(f(x,y),h(x,z))$ satisfies the PDE below:
\begin{equation*}
{\rm{det}}\begin{bmatrix}
F_x&F_y&F_z&0&0&0&0\\
F_{xy}&F_{yy}&F_{yz}&F_y&0&0&0\\
F_{xz}&F_{yz}&F_{zz}&0&F_z&0&0\\
F_{xyz}&F_{yyz}&F_{yzz}&F_{yz}&F_{yz}&0&0\\
F_{xyy}&F_{yyy}&F_{yyz}&2F_{yy}&0&F_y&0\\
F_{xzz}&F_{yzz}&F_{zzz}&0&2F_{zz}&0&F_z\\
F_{xyzz}&F_{yyzz}&F_{yzzz}&F_{yzz}&2F_{yzz}&0&F_{yz}\\
\end{bmatrix}
=0
\end{equation*}
\end{proposition}
The proposition suggests that tree functions are again solutions to (perhaps more tedious) PDEs. It is intriguing to ask if in presence of repeated labels there is a characterization, similar to Theorem \ref{main}, of a TFS as the solution set to a system of PDEs; cf. Question \ref{conjecture}.  We finish with one final application of this idea of transforming a neural network to a tree: In Theorem \ref{nn-size-function-space} of \S\ref{labeled; estimate}, we give an upper bound on the number of bit-valued functions computable by a neural network in terms of parameters depending on the architecture of the network. 

\section{Discussion}\label{discussion}
Here we study the functions that are obtained from hierarchical superpositions; we study functions that can be computed on trees. The hierarchy is  represented by a rooted binary tree where the leaves take different inputs and at each node a bivariate function is applied to outputs from the previous layer. In the setting of analytic functions,  in Theorem \ref{main} we characterize the space of functions that could be generated accordingly as the solution set to a system of PDEs. This characterization enables us to construct examples (e.g. \eqref{non-example}) of functions that could not be implemented on a prescribed tree. This is reminiscent of Minsky's famous XOR theorem \cite{minsky2017perceptrons}.  The space of analytic functions is infinite-dimensional and this  motivates us to investigate two settings in which the TFS is finite-dimensional (polynomials) or even finite (bit-valued). We show that the dimension or size of the TFS is considerably smaller than that of the ambient function space. The number of bit-valued functions could be estimated even for non-binary trees following the same ideas. Finally, we bridge between trees and  neural networks by associating with each feed-forward neural network its corresponding TENN; cf. Figure \ref{fig:tree-expanded-network}. This procedure allows us to apply our analysis of trees yielding an upper bound for the number of bit-valued functions that can be implemented by a neural network. 

Our main result in the continuous setting, Theorem \ref{main}, holds only for analytic functions; see the discussion in \S\ref{smooth setting}. While they constitute a large class of functions, there are important cases where one must deal with continuous non-analytic functions too. For example, a typical deep learning network is built through composition from analytic functions such as linear and sigmoid functions or the hyperbolic tangent; and also, from non-analytic functions such as ReLU or the pooling function. Continuous functions could be approximated locally by analytic ones with any desired precision (although not over the entirety of the real axis).  Therefore, while our main result (local by formulation) is an exact classification, one future direction is to study how well arbitrary continuous functions could be approximated by analytic tree functions. 

The morphology of the dendrites of neurons processes information through a (commonly binary \cite{kollins2005branching, gillette2015topological}) tree and so it is naturally related to the setup of this paper; see Figure \ref{fig:pocket-network} \cite{mel1994information, poirazi2003pyramidal}. A typical dendritic tree receives synaptic inputs from surrounding neurons. When activated, the synapses induce a current which changes the voltage on the dendrite. This is followed by a flow of the resulting current towards (and away from) the root (soma) of the neuron. In typical models of neural physiology, a neuron is segmented into compartments where their  connections and the biological parameters define the dynamics of the voltage for each compartment \cite{segev1998cable}. The dynamics of the electrical activity is often given by the following well-known ODE\footnote{The quantities appeared here are:
\begin{itemize}
 \item $V_i$ is the voltage potential of the $i^{\rm{th}}$ compartment;
 \item $V_{i}^{0,c}$ is the resting voltage potential for the $i^{\rm{th}}$ compartment and the ion $c$;
 \item $C_i$ is the membrane capacitance of the $i^{\rm{th}}$ compartment;
 \item $R_i,R_{i,j}$ denotes the resistance of the $i^{\rm{th}}$ compartment and $R_{i,j}$ is the resistance between $i^{\rm{th}}$ and $j^{\rm{th}}$ compartments; \item $g_c$ is non-linear function corresponding to the ion $c$.
\end{itemize} We only consider currents towards the soma.} \cite{hines1997neuron}:
\begin{equation}
\label{compartment_model}
C_i \frac{dV_i}{dt} = - \frac{V_i}{R_i} + \sum_{j \text{ is a child of }i}{} \frac{V_j - V_i}{R_{ij}} - \sum_{c \in \{\rm{Ca}^{2+},\rm{K}^+,\rm{Na}^+\}} g_{c}(V_i)(V_i-V_i^{0,c}).
\end{equation}
Consequently, in the case of time-varying inputs, TFSs could be of neuroscientific interest. In this situation, the functions at the nodes are operators that receive time-dependent functions as their inputs. Constraints such as \eqref{temp2} in the main theorem may be formulated in this case as well: An operator 
$$\mathbf{F}:(C^\infty)^n\rightarrow C^\infty\quad \mathbf{f}=(f_1,\dots,f_n)\mapsto\mathbf{F}(f_1,\dots,f_n)$$
admitting a tree representation is expected to satisfy equations such as 
\begin{equation}
{\rm{D}}^2_{\mathbf{e}_i,\mathbf{e}_l}\mathbf{F}(\mathbf{f}).{\rm{D}}_{\mathbf{e}_j}\mathbf{F}(\mathbf{f})
={\rm{D}}^2_{\mathbf{e}_j,\mathbf{e}_l}\mathbf{F}(\mathbf{f}).{\rm{D}}_{\mathbf{e}_i}\mathbf{F}(\mathbf{f})
\end{equation}
where derivatives of the operator must be understood in the variational sense
\begin{equation*}
\begin{split}
&{\rm{D}}_{\mathbf{e}_s}\mathbf{F}(\mathbf{f})=\lim_{h\to 0}\frac{\mathbf{F}(\mathbf{f}+h\mathbf{e}_s)-\mathbf{F}(\mathbf{f})}{h};\\
&{\rm{D}}^2_{\mathbf{e}_s,\mathbf{e}_t}\mathbf{F}(\mathbf{f})=\lim_{h\to 0}\frac{\mathbf{F}(\mathbf{f}+h\mathbf{e}_s+h\mathbf{e}_t)-\mathbf{F}(\mathbf{f}+h\mathbf{e}_s)
-\mathbf{F}(\mathbf{f}+h\mathbf{e}_t)+\mathbf{F}(\mathbf{f})}{h^2}.
\end{split}    
\end{equation*}
Utilizing discrete TFSs, in \S\ref{distance} we introduce a metric on the set of labeled binary trees that may be potentially used to quantify how similar two neurons are. A careful adaptation of our results to the time-varying situation could be the object of future enquiries.

Certain assumptions must be made before any application of our treatment of tree functions to the study of neural morphologies. First, from a biological standpoint not all  functions are admissible as functions applied at nodes of neurons. Secondly, the acyclic nature of trees assumes that a neuron functions only due to feed-forward propagation whereas in reality back-propagating action potentials also occur. Thirdly, it is well-known that there are biological mechanisms, such as ephaptic connectivity or neuromodulations, that could affect the computations in a neuron's morphology, and they are not taken into account in typical compartmental models. Our approach only applies to an abstraction of models; however, this abstraction appears meaningful.
\begin{figure}
    \centering
    \includegraphics[width=13cm]{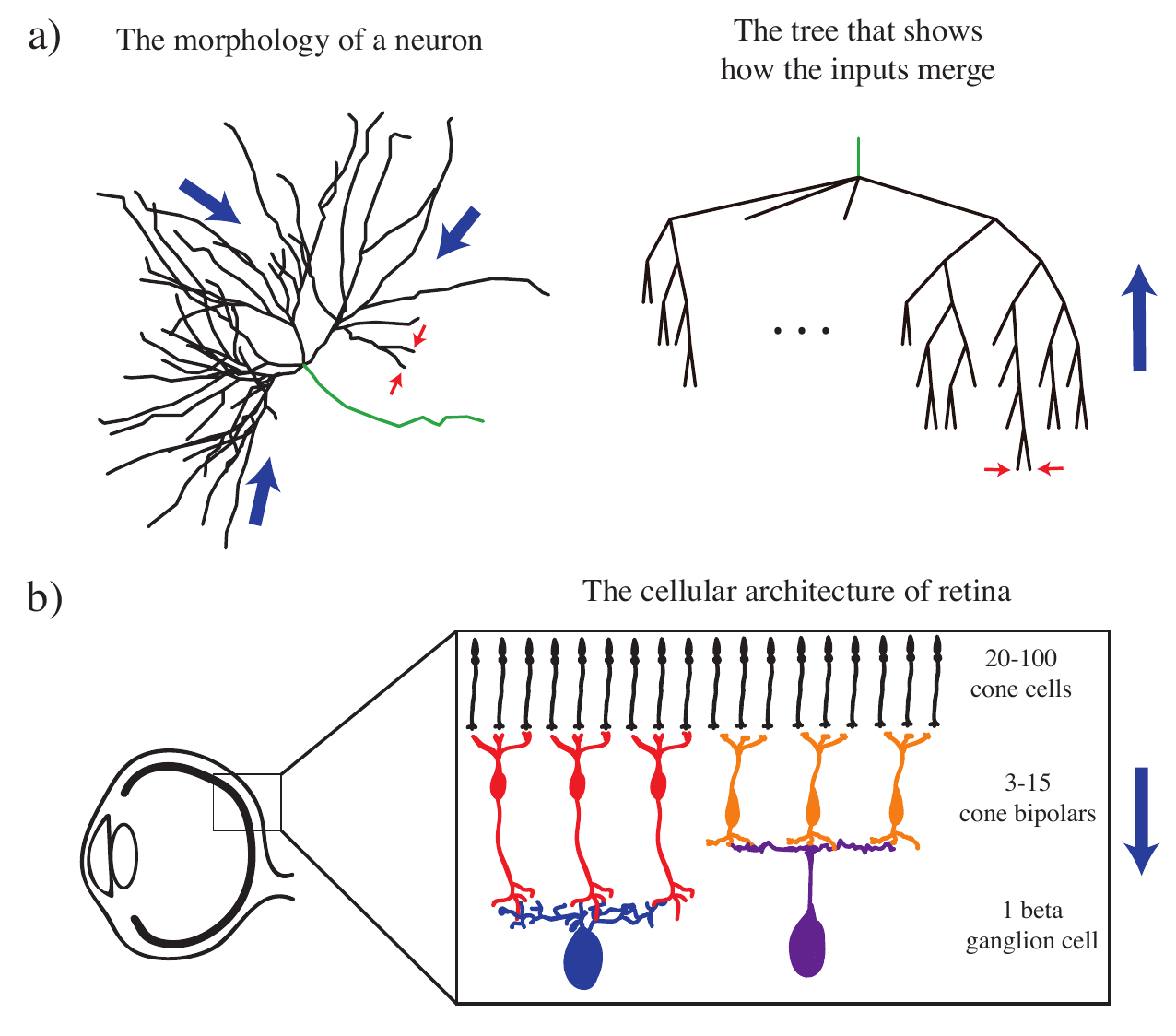}
    \caption{a) Computation in the morphology of a neuron is polarized and can be formulated by a tree function. The inputs to the tree (in red) are combined to produce one output (in green).  The illustrated morphology is a neuron from the human neocortex. The dendrites are in black and the  axon initial segment is in green. (This neuron is adapted from \cite{jacobs1997life,ascoli2007neuromorpho}.) b) Studying the connections between neurons has been a subject of inquiry \cite{ramaswamy2014connectomic,ramaswamy2019hierarchies}. Many neural circuits can also be represented by a tree. One example is provided by connections between neurons in the retina. The illustration is taken from \cite{polyak1941retina}. In both figures the blue arrows show the hierarchy of superpositions.}
    \label{fig:pocket-network}
\end{figure}

In this paper, the ``complexity'' of a TFS in bit-valued and polynomial settings is measured by its cardinality or dimension. However, there are other notions of complexity in the literature that try to capture the capacity of the space of computable functions. For example, the \textit{VC-dimension} measures the expressive power of a space of functions by quantifying the set of separable stimuli \cite{vapnik1994measuring, zador1996vc, mhaskar2017and}. For linear combinations of smooth kernels, the minimum number of bases can measure the complexity of a classifier \cite{bengio2005curse, bengio2006curse}. This suggests that large complexities of shallow networks might be due to the ``amount of variations'' of functions to be computed \cite{kuurkova2016model}. When the functions at the nodes are piece-wise linear, one can count the number of linear regions of the output function \cite{montufar2014number, hanin2019complexity}. The choice of the complexity measurement method is important when it comes to quantifying the difference between two architectures.

When a model is trained on data, we search for the best fit in the function space. Characterizing the landscape of function space can present new methods for training \cite{benjamin2018measuring}. To train models in machine learning, we use a variety of methods such as (stochastic) gradient decent, genetic algorithms, or more recent methods such as learning by coincidence \cite{shaham2018learning}. For some models such as regression, we have explicit formulae that show how to find the parameters from training data. One future line of research is to investigate whether our PDE description of the TFSs can point toward new methods of training.

Since a TFS is much smaller than the ambient space of functions, it is suggestive to consider the approximation by them. In this regard, we fix a target function and take into account the tree functions that approximate it.  Searching for the best approximation of a target function in the function space is realized by the training process. Hence one important question for approximation of a function is the stability of this process \cite{haber2017stable}. Another approach is to develop the mean-field equations to approximate the function space with a fewer equations that are easier to handle \cite{mei2018mean}. Poggio et al. have found a bound for the complexity of a neural networks with smooth (e.g. sigmoid) or non-smooth (e.g. ReLU) non-linearity that provides a prescribed accuracy for functions of a certain regularity \cite{poggio2017and}. Also by estimating the statistical risk of a neural network, one can describe model complexity relative to sample size \cite{barron2018approximation}. Now that we have a description of analytic tree functions as solutions to a system of PDEs, one further direction is to study approximations of arbitrary continuous functions by these solutions.

When the tree function is fed the same input more than once through different leaves, the constraints put on superpositions in Theorem \ref{main} must be refined and become more tedious. In \S\ref{labeled; exact}, we study the simplest possible case, namely the superposition \eqref{3var_form_alternate}. Computing higher derivatives via the chain rule along with a linear algebra argument yield the complicated fourth order PDE of Proposition \ref{pde-for-g(f(x,y),h(x,z))} as a constraint. One future line of research is to formulate similar PDE constraints in the case of general (not necessary binary) trees with repeated labels; cf. Question \ref{conjecture}.  Finding  necessary or sufficient constraints in the repeated regime  would have immediate applications to the study of continuous functions computed via neural networks with this consideration in mind that for the TENN associated with a neural network even functions assigned to nodes are probably repeated. Moreover, in the more specific context of polynomial functions, it is promising to try to formulate results such as Proposition \ref{dimension bound} about the space of polynomial tree functions; or in the bit-valued setting, any strengthening of the bound on the number of bit-valued functions implemented on a general tree that Corollary \ref{bound} provides would be desirable.  

One major goal of the theoretical deep learning is to understand the role of various architectures of neural networks. Previous studies have shown that, compared to shallow networks, deep networks can represent more complex functions \cite{bianchini2014complexity, lin2017does, kuurkova2017probabilistic}. Comparing VC-dimension of different architectures is insightful into why high-dimensional deep networks trained on large training sets often do not seem to show overfit \cite{mhaskar2017and}. Theorem \ref{nn-size-function-space} from the last section of this paper provides further intuition in this direction once instead of more traditional fully connected multi-layer perceptrons, we work with currently more popular sparse neural networks (e.g. convolutional neural networks). This is due to the fact that in the tree expansion of a sparse network the number of children of any arbitrary node would be relatively small. Theorem \ref{nn-size-function-space} indicates that the number of bit-valued functions computable by the network could be large only if the associated tree has numerous leaves. Since the tree is sparse, this could happen only if the depth of the tree (or equivalently, that of the network) is relatively large. The discussion in \S\ref{neural} suggests that studying tree functions could serve as a foundation for interesting theoretical approaches to the study of neural networks.
 
\section{Tree functions}\label{tree functions}
In this section we define the function space associated with a tree in the most general setting.
Suppose we have $n$ inputs (leaves) of a binary tree $T$. We recursively compute the output by applying at each node a function and passing the result to the next level. These calculations continue until we reach the root. 
\begin{definition}\label{definition-functions}
Let $T$ be a  tree and $I$ the set of all possible inputs that a leaf could receive. For any $n \in\mathbb{N}$, suppose $D_n\subseteq\{f|f:I^n\rightarrow I\}$.  The tree function space, $\mathcal{F}(T)$, is defined recursively: $\mathcal{F}(T) = D_1$ 
if $T$ has only one vertex. For larger trees, assuming that the successors of the root of $T$ are the roots of smaller sub-trees $T_1,\dots,T_m$, define: 
\begin{equation*}
\mathcal{F}(T)= \left\{f(F_1, \dots, F_m)|f\in D_m, F_i\in \mathcal{F}(T_i)\right\}
\end{equation*}
\end{definition}
Tree function spaces could be investigated in two different regimes: 
\begin{enumerate}
\item Functions are real analytic, i.e. $D_n = C^\omega(U)$ with $U$ an open subset of $\mathbb{R}^n$ (\S\ref{continuous setting}); 
\item Functions are bit-valued (\S\ref{discrete setting}), i.e. they belong to  
\begin{equation}\label{binary functions}
D_n =  \left\{F:\{0,1\}^n\rightarrow\{0,1\}\right\}.
\end{equation}
\end{enumerate}

 \begin{definition}
A tree is called \textbf{binary} if every non-terminal vertex (every node) of it has precisely two successors; cf. Figure \ref{fig:tree}.
\end{definition}

 \vspace{0.3cm}
  \textbf{The terminology of  binary trees.}
 
\begin{itemize}
\item Root: the unique vertex with no predecessor/parent.
\item Leaf/Terminal: a vertex with no successor\footnote{The reader is cautioned that in our usage of terms such as ``children'', ``parent'', ``successor'' and ``predecessor'' in reference to vertices we have a rooted tree  as illustrated in Figure \ref{fig:tree} in mind where the root precedes every other vertex whereas to implement a function, the computations are done in the ``upward'' direction starting from the leaves in the lowest level and culminating at the root.}/child.
\item Node/Branch point: a vertex which is not a leaf, i.e. has (two) successors. 
\item Sub-tree: all descendants of a vertex along with the vertex itself.
\item Sibling leaves: two leaves with the same parent.
\item Outsider of a triple of leaves: the leaf in the triple which is not a descendant of any common ancestors of the other two. Equivalently, the one which is separated from the other two leaves via a rooted sub-tree; see Figure \ref{fig:predecessor}. For example, in Figure \ref{fig:necessary}, $z$ is the outsider of the triple $\{x,y,z\}$.
\end{itemize}
  \vspace{0.3cm}

\begin{figure}
    \centering
    \includegraphics{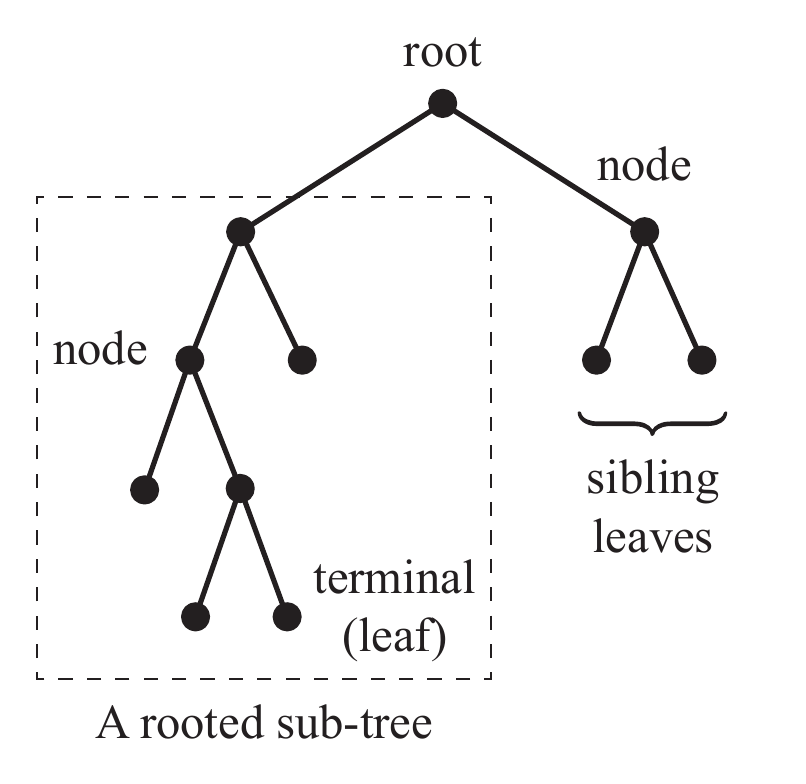}
    \caption{A rooted binary tree with the related terms used throughout this paper.}
    \label{fig:tree}
\end{figure}

\textbf{Convention.} All trees are assumed to be rooted. The number of leaves (terminals) of a tree is always denoted by $n$, and each leaf presents a variable. Unless stated otherwise, the tree is binary and these variables are assumed to be distinct, and hence the corresponding functions are of $n$ variables. Repeated labels come up only in \S\ref{neural}.
\vspace{0.4cm}

\section{Analytic function setting}\label{continuous setting}
\subsection{The case of ternary functions}
\label{three-variable-example} 
In this section, we focus on the first interesting case, namely a binary tree with three inputs. It turns out that the treatment of this basic case and the ideas therein are essential to the proof of Theorem \ref{main}. In order to make one output, two of the inputs should be first combined at one node, and the result of that combination is then combined with the third input at the root. Such functions can be written as:
\begin{equation}
\label{3var_form}
F(x,y,z) = g\left(f(x, y), z\right)
\end{equation}
where $f=f(x,y)$ and $g=g(u,z)$ are two smooth functions of two variables. 

So which functions of three inputs, $F$  could be written as in  \eqref{3var_form}? Taking the derivative w.r.t. $x$ and $y$, we have:
\begin{equation*}
 \frac{\partial F}{\partial x} = \frac{\partial g}{\partial u}.\frac{\partial f}{\partial x}, \quad
 \frac{\partial F}{\partial y} = \frac{\partial g}{\partial u}.\frac{\partial f}{\partial y}
\end{equation*}
Taking the $k^{\rm{th}}$ derivative w.r.t. $z$ yields:
$$
\frac{\partial^{k+1} F}{\partial x \partial z^k}=\frac{\partial^{k+1} g}{\partial u \partial z^k}.\frac{\partial f}{\partial x}, \quad
\frac{\partial^{k+1} F}{\partial y \partial z^k}=\frac{\partial^{k+1} g}{\partial u \partial z^k}.\frac{\partial f}{\partial y}. 
$$
Hence for every $k\in\Bbb{N}$ we should have:

\begin{equation}\label{constraints}
\frac{\partial^{k+1} F}{\partial x \partial z^k}.\frac{\partial F}{\partial y} = \frac{\partial^{k+1} F}{\partial y \partial z^k}.\frac{\partial F}{\partial x};
\end{equation}
as both sides coincide with 
$\frac{\partial^{k+1} g}{\partial u \partial z^k}\frac{\partial g}{\partial u}\frac{\partial f}{\partial x}\frac{\partial f}{\partial y}$.
In particular, for $k=1$ 
\begin{equation*}
\frac{\partial^2 F}{\partial x \partial z}.\frac{\partial F}{\partial y} = \frac{\partial^2 F}{\partial y \partial z}.\frac{\partial F}{\partial x}
\end{equation*}
which is the constraint \eqref{condition} from \S\ref{outline}. 
Notice that the identity is solely based on the function $F$ and serves as a necessary condition for the existence of a presentation such as \eqref{3var_form} for $F$.  

It is essential to observe that constraint \eqref{condition} implies the rest of the constraints imposed on $F$ in \eqref{constraints}. This is trivial for the points where $\left[\frac{\partial F}{\partial x}, \frac{\partial F}{\partial y}\right]=\mathbf{0}$. Otherwise, either $\frac{\partial F}{\partial x}$ or $\frac{\partial F}{\partial y}$ should be non-zero at the point under consideration and hence throughout a small enough neighborhood of it. 
We proceed by induction on $k$. Differentiating \eqref{constraints} w.r.t. $z$ yields:
$$
\frac{\partial^{k+2} F}{\partial x \partial z^{k+1}}.\frac{\partial F}{\partial y}
+\frac{\partial^{k+1} F}{\partial x \partial z^k}.\frac{\partial^2 F}{\partial y\partial z}
=\frac{\partial^{k+2} F}{\partial y \partial z^{k+1}}.\frac{\partial F}{\partial x}
+\frac{\partial^{k+1} F}{\partial y \partial z^k}.\frac{\partial^2 F}{\partial x\partial z}.
$$
We claim that the latter terms of two sides coincide and this will finish the inductive step. From the induction hypothesis 
$\frac{\partial^{k+1} F}{\partial x \partial z^k}.\frac{\partial F}{\partial y}
=\frac{\partial^{k+1} F}{\partial y \partial z^k}.\frac{\partial F}{\partial x}$,
while the base case $k=1$ indicates 
$\frac{\partial^2 F}{\partial x\partial z}\frac{\partial F}{\partial y}
=\frac{\partial^2 F}{\partial y \partial z}\frac{\partial F}{\partial x}$.
The vectors $\left[\frac{\partial^{k+1} F}{\partial x \partial z^k},\frac{\partial^{k+1} F}{\partial y \partial z^k}\right]$
and $\left[\frac{\partial^2 F}{\partial x\partial z}, \frac{\partial^2 F}{\partial y \partial z}\right]$
are multiples of the non-zero vector $\left[\frac{\partial F}{\partial x}, \frac{\partial F}{\partial y}\right]\neq\mathbf{0}$; so they are multiples of each other, i.e. 
$\frac{\partial^{k+1} F}{\partial x \partial z^k}.\frac{\partial^2 F}{\partial y\partial z}=\frac{\partial^{k+1} F}{\partial y \partial z^k}.\frac{\partial^2 F}{\partial x\partial z}$.

In the same vein, identity \eqref{temp2} implies the more general identity below:
\begin{equation}\label{temp1}
\frac{\partial^{k+1} F}{\partial x_i \partial x_l^k}.\frac{\partial F}{\partial x_j} = \frac{\partial^{k+1} F}{\partial x_j \partial x_l^k}.\frac{\partial F}{\partial x_i}
\quad\small
\text{(for all } k\in\Bbb{N} \text{ and } x_i,x_j,x_l \text{ as in Theorem \ref{main})}. 
\normalsize
\end{equation}
This is true even for a greater number of variables $x_{l_1},\dots,x_{l_s}$ in place of $x_l$ in the following sense:
\begin{lemma}\label{only one constraint}
Let $T$ be the tree from Theorem \ref{main} and 
$$F=F(x_1,\dots,x_n)$$ 
be a function of $n$ variables satisfying the constraints \eqref{temp2}. Then 
$$
\frac{\partial^{1+k_1+\dots+k_s} F}{\partial x_i\partial x_{l_1}^{k_1}\dots\partial x_{l_s}^{k_s}}.\frac{\partial F}{\partial x_j} 
=\frac{\partial^{1+k_1+\dots+k_s} F}{\partial x_j\partial x_{l_1}^{k_1}\dots\partial x_{l_s}^{k_s}}.\frac{\partial F}{\partial x_i}
$$
provided that for each leaf $x_{l_t}$ there is a rooted sub-tree containing $x_i,x_j$ that separates them from  $x_{l_t}$ (e.g. Figure \ref{fig:predecessor} where $x_i,x_j$ are separated from $x_l,x_{l'}$ through their predecessor). 
\end{lemma}

\begin{proof}
This can be inferred by employing the same inductive argument; i.e. differentiating the identity
$$
\frac{\partial^{k+1} F}{\partial x_i \partial x_{l_1}^k}.\frac{\partial F}{\partial x_j} = \frac{\partial^{k+1} F}{\partial x_j \partial x_{l_1}^k}.\frac{\partial F}{\partial x_i}
$$
similar to \eqref{temp1} w.r.t. 
variables $x_{l_2},\dots,x_{l_s}$ and using \eqref{temp2} for triples of variables $(x_i,x_j,x_{l_i})$ along with the non-vanishing of one of
$\frac{\partial F}{\partial x_i}$ or $\frac{\partial F}{\partial x_j}$. One trivially gets equality at points where they vanish simultaneously. 
\end{proof}

We next argue that locally, the aforementioned condition is sufficient. In another words, if an analytic ternary function satisfies \eqref{condition}, then it locally admits a tree representation such as \eqref{3var_form}.

\begin{proof}[Proof of Proposition \ref{presentation}]
Let us first impose a mild non-singularity condition at the origin: either $\frac{\partial F}{\partial x}(\mathbf{0})$ or  $\frac{\partial F}{\partial y}(\mathbf{0})$ is non-zero.
Without any loss of generality, we may assume $F(\mathbf{0})=0$ and $\frac{\partial F}{\partial x}(\mathbf{0})\neq 0$. The idea is to come up with a new coordinate system 
\begin{equation}\label{coordinate}
\left(\xi(x,y,z),\eta(x,y,z),z\right)
\end{equation}
centered at the origin in which the function $F$ is dependent on only $\xi,\eta$. Define 
\begin{equation}\label{auxiliary 7}
\xi(x,y,z):=F(x,y,0), \quad \eta(x,y,z):=y.
\end{equation}
The Jacobian of $(\xi,\eta,z)$ w.r.t. $(x,y,z)$ is given by
\begin{equation}\label{Jacobian}
\frac{\partial(\xi,\eta,z)}{\partial(x,y,z)}=
\begin{bmatrix}
\frac{\partial F}{\partial x}(x,y,0)&\frac{\partial F}{\partial y}(x,y,0)&0\\
0&1&0\\
0&0&1
\end{bmatrix}
\end{equation}
whose determinant at the origin is $\frac{\partial F}{\partial x}(\mathbf{0})$ which we have assumed to be non-zero. Thus $(\xi,\eta,z)$ is indeed a coordinate system centered at $\mathbf{0}$. Next, we consider the Taylor expansion of $F(x,y,z)$ w.r.t. $z$:
\begin{equation}\label{expansion}
F(x,y,z)=\sum_{k=0}^\infty\frac{1}{k!}\frac{\partial^k F}{\partial z^k}(x,y,0)z^k;    
\end{equation}
the equality which holds near the origin due to the analyticity assumption. 
We claim that in the new coordinate system $(\xi,\eta,z)$ the partial derivatives $\frac{\partial^k F}{\partial z^k}(x,y,0)$ that appeared above are independent of $\eta,z$. The latter is clear and for the former we apply the chain rule to differentiate with respect to $\eta$:
$$
\frac{\partial}{\partial \eta}\left(\frac{\partial^k F}{\partial z^k}(x,y,0)\right)
=\frac{\partial^{k+1} F}{\partial x\partial z^k}(x,y,0).\frac{\partial x}{\partial \eta}
+\frac{\partial^{k+1} F}{\partial y\partial z^k}(x,y,0).\frac{\partial y}{\partial \eta}.
$$
To calculate $\frac{\partial x}{\partial \eta}, \frac{\partial y}{\partial \eta}$ one has to invert the Jacobian matrix \eqref{Jacobian}:
$$
\frac{\partial (x,y,z)}{\partial(\xi,\eta,z)}=
\begin{bmatrix}
\frac{1}{\frac{\partial F}{\partial x}(x,y,0)}&-\frac{\frac{\partial F}{\partial y}(x,y,0)}{\frac{\partial F}{\partial x}(x,y,0)}&0\\
0&1&0\\
0&0&1
\end{bmatrix}
$$
that yields 
$\frac{\partial x}{\partial \eta}(x,y,z)=-\frac{\frac{\partial F}{\partial y}(x,y,0)}{\frac{\partial F(x,y,0)}{\partial x}}, 
\frac{\partial y}{\partial \eta}(x,y,z)=1$. Plugging in the previous expression for 
$\frac{\partial}{\partial \eta}\left(\frac{\partial^n F}{\partial z^n}(x,y,0)\right)$ we get:
$$
\frac{1}{\frac{\partial F}{\partial x}(x,y,0)}
\left[-\frac{\partial^{k+1} F}{\partial x\partial z^k}(x,y,0).\frac{\partial F}{\partial y}(x,y,0)
+\frac{\partial^{k+1} F}{\partial y\partial z^k}(x,y,0).\frac{\partial F}{\partial x}(x,y,0)\right]
$$
which is zero due to \eqref{constraints}; keep in mind that in a neighborhood of the origin the aforementioned identities are implied by \eqref{condition}; 
cf. Lemma \ref{only one constraint}. We conclude that in \eqref{expansion} each term $\frac{\partial^k F}{\partial z^k}(x,y,0)$
is a function of $\xi(x,y,z)=F(x,y,0)$, e.g. 
$$
\frac{\partial^k F}{\partial z^k}(x,y,0)=g_k\left(F(x,y,0)\right).
$$
Now defining $f(x,y)$ to be $F(x,y,0)$ and $g(u,z)$ to be $\sum_{k=0}^\infty\frac{1}{k!}g_k(u)z^k$, the identity \eqref{expansion}
implies that $F(x,y,z)=g\left(f(x,y),z\right)$ throughout a small enough neighborhood of $\mathbf{0}\in\Bbb{R}^3$.

Next, we omit the assumption that one of the partial derivatives of $F$ is non-zero in Proposition \ref{presentation}. If either  of 
$\frac{\partial^{k+1} F}{\partial x \partial z^k }(\mathbf{0})$ or  $\frac{\partial^{k+1} F}{\partial y \partial z^k }(\mathbf{0})$ is non-zero for some integer $k$, we apply what we just proved to  $\frac{\partial^{k} F}{\partial z^k }$  to get:
\begin{equation}
\frac{\partial^{k} F}{\partial z^k }=g\left(f(x,y),z\right).
\end{equation}
Then integrating $k$ times w.r.t. $z$  provides us with a similar expression for  $F$. There is nothing to prove if all of the partial derivatives 
$\frac{\partial^{k+1} F}{\partial x \partial z^k }(\mathbf{0})$ and  $\frac{\partial^{k+1} F}{\partial y \partial z^k }(\mathbf{0})$
are zero since in that case the Taylor expansion of $F$ describes it as the sum of a function of $(x,y)$ and a function of $z$. 
\end{proof}

\begin{remark}
The idea from the last part of the proof seems to work only for this particular presentation as in general, integration w.r.t. to one of the variables does not preserve forms such as
$g\left(f(x,y),h(z,w)\right)$. Therefore, we are going to need the non-singularity condition of Theorem \ref{main} in the following section. 
\end{remark}

\begin{remark}\label{Frobenius}
An elegant reformulation (from the field of integrable systems) of constraint \eqref{condition} imposed on a smooth tree function $F(x,y,z)$ is to say that the differential form\footnote{The reader may find a very readable account of the theory of differential forms on Euclidean spaces in \cite[chap. 9, sec. 5]{MR1886084}.} 
$\omega:=\frac{\partial F}{\partial x}{\rm{d}}x+\frac{\partial F}{\partial y}{\rm{d}}y$ must satisfy $\omega\wedge{\rm{d}}\omega=0$:
\footnotesize
\begin{equation*}
\begin{split}
\omega\wedge{\rm{d}}\omega
&=\left(\frac{\partial F}{\partial x}{\rm{d}}x+\frac{\partial F}{\partial y}{\rm{d}}y\right)
\wedge\left(\left[\frac{\partial^2 F}{\partial x^2}{\rm{d}}x+\frac{\partial^2 F}{\partial x\partial y}{\rm{d}}y+\frac{\partial^2 F}{\partial x\partial z}{\rm{d}}z\right]\wedge{\rm{d}}x
+\left[\frac{\partial^2 F}{\partial x\partial y}{\rm{d}}x+\frac{\partial^2 F}{\partial y^2}{\rm{d}}y+\frac{\partial^2 F}{\partial y\partial z}{\rm{d}}z\right]\wedge{\rm{d}}y\right)\\
&=\left(\frac{\partial F}{\partial x}{\rm{d}}x+\frac{\partial F}{\partial y}{\rm{d}}y\right)\wedge\left(\frac{\partial^2 F}{\partial x\partial z}{\rm{d}}z\wedge{\rm{d}}x
+\frac{\partial^2 F}{\partial y\partial z}{\rm{d}}z\wedge{\rm{d}}y\right)
=\frac{\partial F}{\partial x}.\frac{\partial^2 F}{\partial y\partial z}{\rm{d}}x\wedge{\rm{d}}z\wedge{\rm{d}}y
+\frac{\partial F}{\partial y}.\frac{\partial^2 F}{\partial x\partial z}{\rm{d}}y\wedge{\rm{d}}z\wedge{\rm{d}}x\\
&=\left(-\frac{\partial F}{\partial x}.\frac{\partial^2 F}{\partial y\partial z}+\frac{\partial F}{\partial y}.\frac{\partial^2 F}{\partial x\partial z}\right){\rm{d}}x\wedge{\rm{d}}y\wedge{\rm{d}}z=0.
\end{split}
\end{equation*}
\normalsize
Similar identities also hold in the general case of a (smooth) tree function $F(x_1,\dots,x_n)\in\mathcal{F}(T)$ as has appeared in Theorem \ref{main}. To any two sibling leaves $x_i,x_j$ assign the differential $1$-form $\omega_{i,j}:=\frac{\partial F}{\partial x_i}{\rm{d}}x_i+\frac{\partial F}{\partial x_j}{\rm{d}}x_j$. A straightforward calculation yields 
$\omega_{i,j}\wedge{\rm{d}}\omega_{i,j}$ as
$$
\omega_{i,j}\wedge{\rm{d}}\omega_{i,j}=\sum_{l\neq i,j}\left(-\frac{\partial F}{\partial x_i}.\frac{\partial^2 F}{\partial x_j\partial x_l}+\frac{\partial F}{\partial x_j}.\frac{\partial^2 F}{\partial x_i\partial x_l}\right){\rm{d}}x_i\wedge{\rm{d}}x_j\wedge{\rm{d}}x_l
$$
which turns out to be zero since any other leaf $x_l$ is an outsider with respect to neighboring $x_i,x_j$; hence the terms inside parentheses vanish due to \eqref{temp2}. 
The non-vanishing condition of Theorem \ref{main} implies that these $1$-forms are linearly independent throughout some small enough open subset of $\Bbb{R}^n$. They define a \textit{differential system} on the aforementioned open subset whose rank is: 
$$n-\#\text{ of pairs of sibling leaves of } T;$$
and the identities $\omega_{i,j}\wedge{\rm{d}}\omega_{i,j}=0$ could be reinterpreted as the \textit{integrability} of this system according to a classical theorem of Frobenius \cite[Theorem 2.11.11]{MR0251745}.  
\end{remark}

The discussion in this subsection settles Theorem \ref{main} for the most basic case of a binary tree with three leaves.

\subsection{Proof of the main theorem}\label{the proof}   
Let $T$ be a binary tree with $n$ leaves as in Theorem \ref{main} and $F$ be a differentiable function of  $n$ variables on an open neighborhood $U$ of $\mathbf{p}\in\Bbb{R}^n$. 

\vspace{0.5cm}
\textbf{\underline{The proof of necessity}}\\
Let $F\in\mathcal{F}(T)$. Consider a triple of variables $(x_i,x_j,x_l)$ as in Theorem \ref{main}. For the ease of notation, suppose they are the first three coordinates $x_1,x_2,x_3$. Given $k\in\Bbb{N}$ and a point 
$$\mathbf{q}=(q_1,q_2,q_3;q_4,\dots,q_n)\in U,$$ 
we need to verify \eqref{temp2} at $\mathbf{q}$. Setting the last $n-3$ coordinates to be constants $q_4,\dots,q_n\,,$ we end up with the function
$$
(x,y,z)\mapsto F(x,y,z;q_4,\dots,q_n)
$$
of three variables defined on the open neighborhood $\pi(U)$ of $(q_1,q_2,q_3)\in\Bbb{R}^3$  which is the image of $U$ under the projection onto the first three coordinates $\pi:\Bbb{R}^n\rightarrow\Bbb{R}^3$. This new function is implemented on a tree with three inputs $x,y,z$ (corresponding to leaves $x_i,x_j,x_l$ in the original statement of Theorem \ref{main}) and with $x,y$ adjacent to the same node as $x_i$ and $x_j$ were separated from $x_l$ in the original tree; see Figure \ref{fig:necessary}. Hence \eqref{condition} holds for this function:
\begin{equation*}
\begin{split}
&\frac{\partial^2 F}{\partial x_1 \partial x_3}(x,y,z;q_4,\dots,q_n).\frac{\partial F}{\partial x_2}(x,y,z;q_4,\dots,q_n) \\
&=\frac{\partial^2 F}{\partial x_2 \partial x_3}(x,y,z;q_4,\dots,q_n).\frac{\partial F}{\partial x_1}(x,y,z;q_4,\dots,q_n);
\end{split}
\end{equation*}
which at $(x,y,z)=(q_1,q_2,q_3)$ yields the desired constraint 
$$
\frac{\partial^2 F}{\partial x_1 \partial x_3}(\mathbf{q}).\frac{\partial F}{\partial x_2}(\mathbf{q}) = 
\frac{\partial^2 F}{\partial x_2 \partial x_3}(\mathbf{q}).\frac{\partial F}{\partial x_1}(\mathbf{q}). \qed
$$

\begin{figure}
    \centering
    \includegraphics[width=\columnwidth]{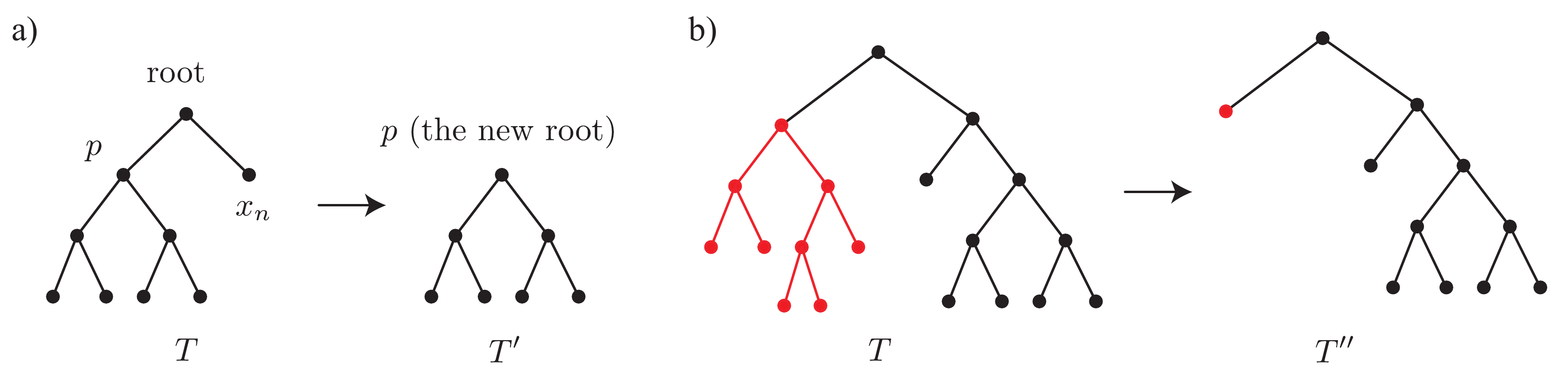}
    \caption{a) Removing a leaf adjacent to the root. b) Collapsing the left rooted sub-tree to a leaf.}
    \label{fig:T''}
\end{figure}

Next, we argue that under the assumptions outlined in the second part of Theorem \ref{main} the identities such as \eqref{temp2} are enough to implement $F(x_1,\dots,x_n)$ on $T$ locally around $\mathbf{p}$. The proof of sufficiency is based on recursively constructing the desired presentation of $F$ as a composition of bivariate functions by reducing the size of $T$. The base of the induction, the case of a tree with three terminals, has already been settled in Proposition \ref{presentation}.\\
\indent We claim that, up to relabeling variables, $F(x_1,\dots,x_n)$ can be written as either 
\begin{equation}\label{first presentation}
F(x_1,\dots,x_{n-1};x_n)=g\left(f(x_1,\dots,x_{n-1}),x_n\right)
\end{equation}
or 
\begin{equation}\label{second presentation}
F(x_1,\dots,x_s;x_{s+1},\dots,x_n)=g\left(f(x_1,\dots,x_s),x_{s+1},\dots,x_n\right)
\end{equation}
where the function $f$ satisfies the hypothesis of the existence part of Theorem \ref{main} for $n-1$ or $s$ variables (the integer $s\in\{2,\dots,n-2\}$ is going to be introduced shortly). In terms of the tree $T$, the first normal form occurs when $x_n$ is connected directly to the root; the removal of the leaf and the root then results in a smaller tree $T'$ with $n-1$ leaves; cf. part (a) of Figure \ref{fig:T''}. The induction hypothesis then establishes $f\in\mathcal{F}(T')$ and finishes the proof. On the other hand, \eqref{second presentation} comes up when neither of the two rooted sub-trees obtained from excluding the root is singleton. The number $s\geq 2$ here denotes the number of the leaves of one of these sub-trees, e.g. the ``left'' one. By symmetry, let us assume that variables are labeled such that $x_1,\dots,x_s$ are the leaves of the sub-tree to the left of the root while $x_{s+1},\dots,x_n$ appear in the sub-tree to the right. Graph-theoretically, gathering the variables $x_1,\dots,x_s$ together in \eqref{second presentation} amounts to collapse the left sub-tree to a leaf. This results in a new binary tree $T''$ with the same root but with $n-s+1$ leaves which is of the form discussed before: it has a ``top'' leaf directly connected to the root; see part (b) of Figure \ref{fig:T''}. The final step is to invoke the induction hypothesis to argue that $g$ in \eqref{second presentation} belongs to $\mathcal{F}(T'')$.

\vspace{0.5cm}
\textbf{\underline{Part I of the proof of sufficiency: Suppose there is a leaf adjacent to the root.}}\\
Without loss of generality, we consider everything in a neighborhood of $\mathbf{0}\in\Bbb{R}^n$ and assume $F(\mathbf{0})=0$. 
Theorem \ref{main} also requires at least one of the partial derivatives of $F$ w.r.t. $x_1,\dots,x_{n-1}$ to be non-zero; by symmetry, let us assume 
$\frac{\partial F}{\partial x_1}(\mathbf{0})\neq 0$.

Next, define 
\begin{equation}\label{coordinate 1}
(y_1,y_2,\dots,y_n)=\left(F(x_1,\dots,x_{n-1},0),x_2,\dots,x_n\right).    
\end{equation}
This is a new coordinate system centered at the origin as the Jacobian 
\small
\begin{equation}\label{Jacobian 1}
\begin{split}
&\frac{\partial(y_1,y_2,\dots,y_n)}{\partial(x_1,x_2,\dots,x_n)}=\\
&\begin{bmatrix}
\frac{\partial F}{\partial x_1}(x_1,\dots,x_{n-1},0)&\frac{\partial F}{\partial x_2}(x_1,\dots,x_{n-1},0)&\dots&\frac{\partial F}{\partial x_{n-1}}(x_1,\dots,x_{n-1},0)&0\\
0&1&\dots&0&0\\
\vdots&\vdots&\ddots&\vdots&0\\
0&0&\dots&1&0\\
0&0&\dots&0&1
\end{bmatrix}
\end{split}
\end{equation}
\normalsize
is of determinant $\frac{\partial F}{\partial x_1}(\mathbf{0})\neq 0$ at the origin. The goal is to write down $F(x_1,\dots,x_n)$ in a form 
\begin{equation}\label{alternative alternative presentation}
F(x_1,\dots,x_{n-1};x_n)=g\left(F(x_1,\dots,x_{n-1},0),x_n\right)
\end{equation}
similar to \eqref{first presentation} for a suitable bivariate function $g$ and then applying the induction hypothesis to 
$F(x_1,\dots,x_{n-1},0)$ which of course satisfies \eqref{temp2} with the original tree $T$ replaced with $T'$. To this end, we consider the Taylor expansion of $F(x_1,x_2,\dots,x_n)$ w.r.t. $x_n$:
\begin{equation}\label{generalized expansion}
F(x_1,\dots,x_{n-1};x_n)=\sum_{k=0}^\infty\frac{1}{k!}\frac{\partial^kF}{\partial x_n^k}(x_1,\dots,x_{n-1},0)x_n^k.    
\end{equation}
We claim that the functions
$$(x_1,\dots,x_{n-1},x_n)\mapsto\frac{\partial^kF}{\partial x_n^k}(x_1,\dots,x_{n-1},0)$$
appeared as coefficients are dependent only on the first component of the new coordinate system \eqref{coordinate 1}, or in other words
$$
\forall 2\leq i\leq n: \frac{\partial }{\partial y_i}\left(\frac{\partial^kF}{\partial x_n^k}(x_1,\dots,x_{n-1},0)\right)=0. 
$$
 This is immediate when $i=n$ as we are basically differentiating w.r.t. $x_n$. For $2\leq i<n$, we need to apply the chain rule to get
$$
\frac{\partial }{\partial y_i}\left(\frac{\partial^kF}{\partial x_n^k}(x_1,\dots,x_{n-1},0)\right)=
\sum_{s=1}^{n}\frac{\partial^{k+1}F}{\partial x_s\partial x_n^k}(x_1,\dots,x_{n-1},0)\frac{\partial x_s}{\partial y_i};
$$
where the partial derivatives $\frac{\partial x_s}{\partial y_i}$ are entries of the inverse of \eqref{Jacobian 1} given by
\small
\begin{equation*}
\begin{split}
&\frac{\partial(x_1,\dots,x_n)}{\partial(y_1,\dots,y_n)}=\\
&\begin{bmatrix}
\frac{1}{\frac{\partial F}{\partial x_1}(x_1,\dots,x_{n-1},0)}&-\frac{\frac{\partial F}{\partial x_2}(x_1,\dots,x_{n-1},0)}{\frac{\partial F}{\partial x_1}(x_1,\dots,x_{n-1},0)}&\dots&
-\frac{\frac{\partial F}{\partial x_{n-1}}(x_1,\dots,x_{n-1},0)}{\frac{\partial F}{\partial x_1}(x_1,\dots,x_{n-1},0)}&0\\
0&1&\dots&0&0\\
\vdots&\vdots&\ddots&\vdots&0\\
0&0&\dots&1&0\\
0&0&\dots&0&1
\end{bmatrix}.
\end{split}
\end{equation*}
\normalsize

Hence $\frac{\partial x_s}{\partial y_i}$ is non-zero only for $s=1,i$ and 
\begin{equation}
\begin{split}
&\frac{\partial }{\partial y_i}\left(\frac{\partial^kF}{\partial x_n^k}(x_1,\dots,x_{n-1},0)\right)=\\
&\frac{1}{\frac{\partial F}{\partial x_1}(x_1,\dots,x_{n-1},0)}
\Big[-\frac{\partial^{k+1}F}{\partial x_1\partial x_n^k}(x_1,\dots,x_{n-1},0)\frac{\partial F}{\partial x_i}(x_1,\dots,x_{n-1},0)\\
&+\frac{\partial^{k+1}F}{\partial x_i\partial x_n^k}(x_1,\dots,x_{n-1};0)\frac{\partial F}{\partial x_1}(x_1,\dots,x_{n-1},0)\Big];
\end{split}    
\end{equation}
which is zero as \eqref{temp1} holds (keep in mind that the sub-tree $T'$ of $T$ has 
$x_1,x_i$ while it misses $x_n$ since, as part (a) of Figure \ref{fig:T''} demonstrates,  the leaf corresponding to $x_n$ is connected directly to the root of $T$.). Consequently, the term 
$\frac{\partial^k F}{\partial x_n^k}(x_1,\dots,x_{n-1},0)$ in \eqref{generalized expansion}  can be written as a function $h_k(y_1)$ 
where $y_1$ has been defined to be
$F(x_1,\dots,x_{n-1},0)$ in \eqref{coordinate 1}. Hence 
$$
g(u,v):=\sum_{k=0}^\infty\frac{1}{k!}h_k(u)v^k
$$
works in \eqref{alternative alternative presentation}. \qed

\vspace{0.5cm}
\textbf{\underline{Part II of the proof of sufficiency: Suppose there is no leaf adjacent to the root.}}\\
We work with the convention discussed above: Among the two rooted sub-trees resulting from excluding the root of $T$, the ``left'' one has variables $x_1,\dots,x_s$ as its leaves while the ``right'' one has the rest of the variables $x_{s+1},\dots,x_n$. The number $s$ is assumed to be larger than $1$ as the case of $s=1$ has just been treated in Part I of the proof of sufficiency.    

Expand $F(x_1,\dots,x_s;x_{s+1},\dots,x_n)$ w.r.t. the last $n-s$ variables:
\small
\begin{equation}\label{nvar-expansion} 
\begin{split}
&F(x_1,\dots,x_s;x_{s+1},\dots,x_n)\\
&=\sum_{k_{s+1},\dots,k_n\geq 0}\frac{1}{k_{s+1}!\dots k_n!}
\frac{\partial^{k_{s+1}+\dots+k_n} F}{\partial x_{s+1}^{k_{s+1}}\dots\partial x_n^{k_n}}(x_1,\dots,x_s;\overbrace{0,\dots,0}^{n-s\text{ times}}) 
x_{s+1}^{k_{s+1}}\dots x_n^{k_n}.  
\end{split}
\end{equation}
\normalsize
The non-vanishing assumption of Theorem \ref{main} requires the partial derivative at $\mathbf{p}=\mathbf{0}$ of $F$ with respect to at least one of the variables in the left sub-tree to be non-zero; let us assume $\frac{\partial F}{\partial x_1}(\mathbf{0})\neq 0$.
By a similar Jacobian determinant computation appeared in part I of the proof, the assignment
\begin{equation}\label{auxiliary 1}
(x_1,\dots,x_s)\mapsto\left(\xi(x_1,\dots,x_s):=F(x_1,\dots,x_s;\overbrace{0,\dots,0}^{n-s\text{ times}}),x_2,\dots,x_s\right)
\end{equation}
defines a coordinate system centered at the origin of $\Bbb{R}^s$ as $\frac{\partial F}{\partial x_1}(\mathbf{0})\neq 0$. Repeating the argument that has come up multiple times before, the term 
$$
\frac{\partial^{k_{s+1}+\dots+k_n} F}{\partial x_{s+1}^{k_{s+1}}\dots\partial x_n^{k_n}}(x_1,\dots,x_s;\overbrace{0,\dots,0}^{n-s\text{ times}})
$$
appeared in \eqref{nvar-expansion} is a function of $\xi(x_1,\dots,x_s)$ since for every $2\leq i\leq s$ its derivative w.r.t. the component $x_i$ of the system \eqref{auxiliary 1} is zero due to 
$$
\frac{\partial^{1+k_{s+1}+\dots+k_n} F}{\partial x_1\partial x_{s+1}^{k_{s+1}}\dots\partial x_n^{k_n}}.\frac{\partial F}{\partial x_i}
=\frac{\partial^{1+k_{s+1}+\dots+k_n} F}{\partial x_i\partial x_{s+1}^{k_{s+1}}\dots\partial x_n^{k_n}}.\frac{\partial F}{\partial x_1};
$$
the identity that follows from Lemma \ref{only one constraint} because the right sub-tree separates $x_{s+1},\dots,x_n$ from the leaves $x_1,x_i$ of the left sub-tree. Therefore, \eqref{nvar-expansion} can be rewritten as
\small
\begin{equation}\label{auxiliary 2}
\begin{split}
&F(x_1,\dots,x_s;x_{s+1},\dots,x_n)\\
&=\sum_{k_{s+1},\dots,k_n\geq 0}\frac{1}{k_{s+1}!\dots k_n!}
u_{k_{s+1},\dots,k_n}\left(\xi(x_1,\dots,x_s)\right)x_{s+1}^{k_{s+1}}\dots x_n^{k_n};
\end{split}
\end{equation}
\normalsize
where the single variable function $u_{k_{s+1},\dots,k_n}(\xi)$ satisfies
\begin{equation}\label{auxiliary 3}
\frac{\partial^{k_{s+1}+\dots+k_n} F}{\partial x_{s+1}^{k_{s+1}}\dots\partial x_n^{k_n}}(x_1,\dots,x_s;\overbrace{0,\dots,0}^{n-s\text{ times}})
=u_{k_{s+1},\dots,k_n}\left(\xi(x_1,\dots,x_s)\right).     
\end{equation}

The goal is to show that the function 
\begin{equation}\label{auxiliary 4}
g(\xi;x_{s+1},\dots,x_n):=\sum_{k_{s+1},\dots,k_n\geq 0}\frac{1}{k_{s+1}!\dots k_n!}u_{k_{s+1},\dots,k_n}(\xi)x_{s+1}^{k_{s+1}}\dots x_n^{k_n}
\end{equation}
of $n-s+1$ variables can be represented by the tree $T''$ as in  that case 
$$F(x_1,\dots,x_s;x_{s+1},\dots,x_n)$$
is given by
$$
g\left(\xi(x_1,\dots,x_s),x_{s+1},\dots,x_n\right)
$$
which is in the form of \eqref{second presentation} with $\xi(x_1,\dots,x_s)=F(x_1,\dots,x_{s};\overbrace{0,\dots,0}^{n-s\text{ times}})$ obviously satisfying the induction hypothesis for the sub-tree to the left of the root and $g$ satisfying the same but for the tree $T''$ with $n-s+1$ terminals obtained from collapsing the aforementioned left sub-tree of $T$ to a point. To verify the conditions of the theorem for $g$, first observe that for $i,j,l>s$ the identity
$$
\frac{\partial^2g}{\partial x_i\partial x_l}.\frac{\partial g}{\partial x_j}
=\frac{\partial^2g}{\partial x_j\partial x_l}.\frac{\partial g}{\partial x_i}
$$
holds trivially. This is due to the fact that by definition 
\begin{equation}\label{auxiliary 5'}
g\left(\xi(x_1,\dots,x_s);x_{s+1},\dots,x_n\right)=F(x_1,\dots,x_s;x_{s+1},\dots,x_n);
\end{equation}
which yields 
\small
\begin{equation}\label{auxiliary 6}
\frac{\partial g}{\partial x_t}\left(\xi(x_1,\dots,x_s);x_{s+1},\dots,x_n\right)=\frac{\partial F}{\partial x_t}(x_1,\dots,x_s;x_{s+1},\dots,x_n) 
\end{equation}
\normalsize
for any $s+1 \leq t\leq n$ and besides, $F$ satisfies the analogous constraint
$$
\frac{\partial^2 F}{\partial x_i\partial x_l}.\frac{\partial F}{\partial x_j}
=\frac{\partial^2F}{\partial x_j\partial x_l}.\frac{\partial F}{\partial x_i}.
$$
It needs to be mentioned that furthermore, the non-vanishing requirement of the induction hypothesis can be deduced from \eqref{auxiliary 6} since 
$$
\frac{\partial g}{\partial x_i}(\mathbf{0})=\frac{\partial F}{\partial x_i}(\mathbf{0})
$$
for every $i>s$; and any leaf $x_i$ of the new tree $T''$ having a sibling leaf $x_{i'}$ has to come from the original tree $T$; hence the desired 
$$\frac{\partial g}{\partial x_i}(\mathbf{0})\text{ or }\frac{\partial g}{\partial x_{i'}}(\mathbf{0})\neq 0$$
is the same as the non-vanishing condition 
$$\frac{\partial F}{\partial x_i}(\mathbf{0})\text{ or }\frac{\partial F}{\partial x_{i'}}(\mathbf{0})\neq 0$$
that the theorem has imposed on $F$. Consequently, the challenging part would be to verify 
\begin{equation}\label{auxiliary 4'}
\frac{\partial^2 g}{\partial\xi\partial x_{i}}.\frac{\partial g}{\partial x_j}
=\frac{\partial^2 g}{\partial\xi\partial x_j}.\frac{\partial g}{\partial x_{i}};
\end{equation}
for any $s<i<j\leq n$. Differentiating \eqref{auxiliary 5'} along with the definition of $\xi(x_1,\dots,x_s)$ as $F(x_1,\dots,x_s;\overbrace{0,\dots,0}^{n-s\text{ times}})$ yield
\footnotesize
$$
\frac{\partial F}{\partial x_1}(x_1,\dots,x_s;x_{s+1},\dots,x_n)
=\frac{\partial g}{\partial\xi}\left(\xi(x_1,\dots,x_s);x_{s+1},\dots,x_n\right).\frac{\partial F}{\partial x_1}(x_1,\dots,x_s;\overbrace{0,\dots,0}^{n-s\text{ times}}).
$$
\normalsize
In particular, near the origin where $\frac{\partial F}{\partial x_1}\neq 0$, one has 
\small
\begin{equation}\label{auxiliary 5}
\frac{\partial g}{\partial\xi}\left(\xi(x_1,\dots,x_s);x_{s+1},\dots,x_n\right)=
\frac{\frac{\partial F}{\partial x_1}(x_1,\dots,x_s;x_{s+1},\dots,x_n)}{\frac{\partial F}{\partial x_1}(x_1,\dots,x_s;0,\dots,0)}.
\end{equation}
\normalsize
Taking the partial derivatives of \eqref{auxiliary 5} w.r.t. variables $x_i,x_j$ (where $i,j>s$) and invoking \eqref{auxiliary 6}, we see that \eqref{auxiliary 4'} amounts to
\begin{equation*}
\begin{split}
&\frac{\frac{\partial^2 F}{\partial x_1\partial x_i}(x_1,\dots,x_s;x_{s+1},\dots,x_n)}{\frac{\partial F}{\partial x_1}(x_1,\dots,x_s;0,\dots,0)}.\frac{\partial F}{\partial x_j}(x_1,\dots,x_s;x_{s+1},\dots,x_n)\\
&=\frac{\frac{\partial^2 F}{\partial x_1\partial x_j}(x_1,\dots,x_s;x_{s+1},\dots,x_n)}{\frac{\partial F}{\partial x_1}(x_1,\dots,x_s;0,\dots,0)}.\frac{\partial F}{\partial x_i}(x_1,\dots,x_s;x_{s+1},\dots,x_n);
\end{split}
\end{equation*}
which holds since
$$
\frac{\partial^2 F}{\partial x_1\partial x_i}.\frac{\partial F}{\partial x_j}
=\frac{\partial^2F}{\partial x_1\partial x_j}.\frac{\partial F}{\partial x_i};
$$
that is, one of the constraints imposed on $F$ in \eqref{temp2} ($x_1$ is separated from $x_i,x_j$ via the right sub-tree). \qed

\subsection{Smooth function setting}\label{smooth setting}
The proof above heavily relied on the analyticity of the function $F$ under consideration. As a matter of fact,  in the context of $C^\infty$ functions one needs to strengthen the constraint \eqref{temp2} as follows (for simplicity, we have replaced $x_i$, $x_j$ and $x_l$ with $x_1$,$x_2$ and $x_3$ respectively):
\begin{equation}\label{strong constraint}
\begin{split}
&\frac{\partial^2 F}{\partial x_1\partial x_3}(a,b,c;q_4,\dots,q_n).\frac{\partial F}{\partial x_2}(a,b,c';q_4,\dots,q_n)\\
&=\frac{\partial^2 F}{\partial x_2\partial x_3}(a,b,c;q_4,\dots,q_n).\frac{\partial F}{\partial x_1}(a,b,c';q_4,\dots,q_n);     
\end{split}
\end{equation}
\normalsize
for any two points $(a,b,c;q_4,\dots,q_n)$ and $(a,b,c';q_4,\dots,q_n)$ lying in an open ball on which $F$ is defined. This is best demonstrated for the superposition 
$$F(x,y,z) = g\left(f(x, y), z\right)$$
from \eqref{3var_form}: The computations carried out there give us 
\small
\begin{equation*}
\begin{split}
&\frac{\partial F}{\partial x}(a,b,c')=\frac{\partial g}{\partial u}\left(f(a,b),c'\right).\frac{\partial f}{\partial x}(a,b),\quad
\frac{\partial F}{\partial y}(a,b,c')=\frac{\partial g}{\partial u}\left(f(a,b),c'\right).\frac{\partial f}{\partial y}(a,b);\\
&\frac{\partial^2 F}{\partial x\partial z}(a,b,c)=\frac{\partial^2 g}{\partial u\partial z}\left(f(a,b),c\right).\frac{\partial f}{\partial x}(a,b),\quad
\frac{\partial^2 F}{\partial y\partial z}(a,b,c)=\frac{\partial^2 g}{\partial u\partial z}\left(f(a,b),c\right).\frac{\partial f}{\partial y}(a,b);
\end{split}
\end{equation*}
\normalsize
which readily implies
$$
\frac{\partial^2 F}{\partial x\partial z}(a,b,c).\frac{\partial F}{\partial y}(a,b,c')=\frac{\partial^2 F}{\partial y\partial z}(a,b,c).\frac{\partial F}{\partial x}(a,b,c').
$$
Notice that the stronger constraint \eqref{strong constraint} could be derived from the ordinary ones \eqref{temp2} and \eqref{temp1} if the function is analytic: Form the single variable function 
\begin{equation*}
\begin{split}
z\mapsto &\frac{\partial^2 F}{\partial x_1\partial x_3}(a,b,z;q_4,\dots,q_n).\frac{\partial F}{\partial x_2}(a,b,c';q_4,\dots,q_n)\\
&-\frac{\partial^2 F}{\partial x_2\partial x_3}(a,b,z;q_4,\dots,q_n).\frac{\partial F}{\partial x_1}(a,b,c';q_4,\dots,q_n)
\end{split}
\end{equation*}
defined on an open interval of $z$-values containing both $c$ and $c'$. The function vanishes at $z=c'$ and we wish to show that it is identically zero as then it would be zero for $z=c$ too. Due to analyticity, it suffices to show that the derivatives of all orders vanish at $z=c'$. Notice that the $k^{\rm{th}}$ derivative at that point is   
\begin{equation*}
\begin{split}
&\frac{\partial^{k+1} F}{\partial x_1\partial x_3^k}(a,b,c';q_4,\dots,q_n).\frac{\partial F}{\partial x_2}(a,b,c';q_4,\dots,q_n)\\
&-\frac{\partial^{k+1} F}{\partial x_2\partial x_3^k}(a,b,c';q_4,\dots,q_n).\frac{\partial F}{\partial x_1}(a,b,c';q_4,\dots,q_n)
\end{split}
\end{equation*}
which is zero due to \eqref{temp1}. We are going to continue to work in the analytic category hereafter and so there would be no need to generalize  constraint \eqref{temp2} to identities such as \eqref{strong constraint}. 

\begin{example}
A typical example of a non-analytic smooth function is 
$$
\rho(z):=
\begin{cases}
{\rm{e}}^{-\frac{1}{z}}\,\, z>0\\
0\,\quad\text{otherwise}
\end{cases}.
$$
We are going to use it to construct a smooth (but of course non-analytic) function of three variables $F(x,y,z)$ that satisfies  constraint \eqref{condition} but not the generalized one introduced above. Set
$$
F(x,y,z)=x\rho(z)+y\rho(-z).
$$
We have
$$
\frac{\partial^2 F}{\partial x\partial z}.\frac{\partial F}{\partial y}=\rho'(z)\rho(-z)
$$
while 
$$\frac{\partial^2 F}{\partial y\partial z}.\frac{\partial F}{\partial x}=-\rho'(-z)\rho(z);$$
which coincide since they are both identically zero due to the fact that $\rho$, and therefore $\rho'$, vanishes at non-positive numbers. Notice that the generalized condition is not satisfied here:
$$
\frac{\partial^2 F}{\partial x\partial z}(x,y,z).\frac{\partial F}{\partial y}(x,y,-z)=\rho'(z)\rho(z)  
$$
is positive when $z>0$ whereas
$$
\frac{\partial^2 F}{\partial y\partial z}(x,y,z).\frac{\partial F}{\partial x}(x,y,-z)=-\rho'(-z)\rho(-z)  
$$
is zero in that case.
It should not be surprising that in the absence of analyticity, $F(x,y,z)$ serves as a counter example to Proposition \ref{presentation}: We are going to argue that there is no representation 
$$
F(x,y,z)=g\left(f(x,y),z\right)
$$
of $F$ in a neighborhood of the origin with $g,h$ continuous functions of two variables. Aiming for a contradiction, suppose there are continuous functions $f,g$ satisfying
$$
g\left(f(x,y),z\right)=x\rho(z)+y\rho(-z)
$$
for $(x,y,z)\in[-\epsilon,\epsilon]^3$ where $\epsilon>0$. Plugging $z=\epsilon$ and $z=-\epsilon$, we arrive at:
$$
g\left(f(x,y),\epsilon\right)=x{\rm{e}}^{-\frac{1}{\epsilon}}, \quad g\left(f(x,y),-\epsilon\right)=y{\rm{e}}^{-\frac{1}{\epsilon}}.
$$
In conjunction, these two identities imply that $f:[-\epsilon,\epsilon]^2\rightarrow\Bbb{R}$ is injective. This is absurd as for obvious topological reasons, there is no continuous injective map from any non-degenerate square to the real line. 
\end{example}

\subsection{Reducing the number of constraints}
\label{codimension}
Theorem \ref{main} provides necessary and sufficient conditions for a multivariate function to belong to the function space associated with a tree. Here, we approach the problem of finding the co-dimension of this infinite-dimensional subspace by counting the number of independent constraints. In general, the condition \eqref{condition} of Theorem \ref{main} should hold for any triple of variables and therefore, ${n \choose 3}$ equations for a tree with $n$ leaves. However, many of these equations are redundant. To find the number of algebraically independent equations we shall need the following lemma.

\begin{lemma}\label{number of constraints}
Let $T$ be a binary tree. Denote its left and right sub-trees by $T_1,T_2$ and suppose that they have $s$ and $n-s$ leaves. Then the number of algebraically independent equations corresponding to triples that have elements from both $T_1$ and $T_2$ is  $s(n-s)-1$.
\end{lemma}
\begin{proof}
Denote the leaves of $T_1$ by $x_1,\dots,x_s$ and those of $T_2$ by $x_{s+1},\dots,x_n$. For the ease of notation, we switch to the subscript notation for partial derivatives. If $x_i$ is another leaf of $T_1$ (different from $x_1$) and $x_j$ is another leaf  for $T_2$ (different from $x_n$), then:
\begin{equation}\label{minimum-dim}
\frac{F_{x_ix_j}}{F_{x_i}F_{x_j}}=\frac{F_{x_1x_j}\frac{F_{x_i}}{F_{x_1}}}{F_{x_i}F_{x_j}}=\frac{F_{x_1x_j}}{F_{x_1}F_{x_j}}=\frac{F_{x_1x_n}\frac{F_{x_j}}{F_{x_n}}}{F_{x_1}F_{x_j}}
=\frac{F_{x_1x_n}}{F_{x_1}F_{x_n}}.
\end{equation}
Here, we have used the constraint \eqref{temp2} for triples $x_1,x_i,x_j$ and $x_1,x_j,x_n$ where in the former (resp. the latter) two of the variables are in the sub-tree $T_1$ (resp. $T_2$) while the third one is in the other sub-tree.  
As $(x_i,x_j)$ varies among all the pairs formed by the leaves $x_i$ of $T_1$ and  $x_j$ of $T_2$, we get $s(n-s)-1$ constraints as we need
$s(n-s)$ fractions of the form $\frac{F_{x_ix_j}}{F_{x_i}F_{x_j}}$ to coincide. 
\end{proof} 

\begin{proposition}\label{binomial growth}
There are ${n-1 \choose 2}$ algebraically independent constraints for a  tree $T$ with $n$ leaves.
\end{proposition}

\begin{proof}
This is clear for the first few values of $n$ as the number of constraints  is zero when $n=1,2$ and is $\binom{3}{3}=1$ for $n=3$. Fix a tree $T$ with $n\geq 3$ leaves and denote its left and right sub-trees by $T_1,T_2$. Suppose they have $s,n-s$ leaves respectively. By induction we know that there are ${s-1 \choose 2}$ and ${n-s-1 \choose 2}$ algebraically independent equations corresponding to sub-trees $T_1$ and $T_2$  respectively. The previous lemma proves that there are exactly $s(n-s)-1$ independent constraints coming from triples that have indices from both $T_1,T_2$. Putting them together with the aforementioned constraints yield the number of algebraically independent constraints  for $T$ as:
\begin{equation*}
\begin{split}
&{s-1 \choose 2} + {n-s-1 \choose 2} + s(n-s)-1\\
&={n-1 \choose 2}.
\end{split}
\end{equation*}
\end{proof}

\begin{example}\label{4var-1}
For the first tree with four terminals in Figure \ref{fig:three-four} we have:
\begin{equation}\label{four-2}
\begin{split}
& F_{xz}F_y = F_{yz}F_x,\\
& F_{xw}F_y = F_{yw}F_x,\\
& F_{zw}F_x = F_{xw}F_z,\\
& F_{zw}F_y = F_{yw}F_z.
\end{split}
\end{equation}
while  for the second one:
\begin{equation}\label{four-1}
\begin{split}
& F_{xz}F_y = F_{yz}F_x,\\
& F_{xw}F_y = F_{yw}F_x,\\
& F_{xz}F_w = F_{xw}F_z,\\
& F_{yz}F_w = F_{yw}F_z.
\end{split}
\end{equation}
These group of four equations may be rewritten as:
\begin{equation*}
\frac{F_{xz}}{F_xF_z}=\frac{F_{yz}}{F_yF_z}\quad  \frac{F_{xw}}{F_xF_w}=\frac{F_{yw}}{F_yF_w}=\frac{F_{zw}}{F_zF_w}
\end{equation*}
and
\begin{equation*}
\frac{F_{xz}}{F_xF_z}=\frac{F_{yz}}{F_yF_z}=\frac{F_{xw}}{F_xF_w}=\frac{F_{yw}}{F_yF_w}
\end{equation*}
respectively. Thus we see that a lesser number of equations suffices.
\end{example}

\begin{remark}
Proposition \ref{binomial growth} should be understood in a \textit{generic} sense. To elaborate, it indicates that for an $n$-variate function $F(x_1,\dots,x_n)$ one can reduce the total number of ${n \choose 3}$ constraints to ${n-1 \choose 2}$; but, in view of the division took place in \eqref{minimum-dim}, one needs to have $F_{x_i}\not\equiv 0$ for every $i$; that is, $F$ should not be independent of any of its variables. This non-vanishing condition is generic: it is open (i.e. persists under small perturbations) and the locus where it fails 
is of positive co-dimension as it is determined by the union of non-trivial functional equations $F_{x_i}\equiv 0$.  Such a non-vanishing requirement is necessary because, as a matter of fact, for an arbitrary function $F$ one cannot ignore any of the ${n \choose 3}$ constraints of the form $F_{x_ix_l}F_{x_j}=F_{x_jx_l}F_{x_i}$: Motivated by the non-example \eqref{non-example}, notice that 
$$
F(x_1,\dots,x_n):=x_i+x_j+x_l+x_ix_jx_l
$$
does not satisfy the preceding constraint but satisfies other ones since the mixed second order partial derivative of $F$ w.r.t. any pair of variables other than $(x_i,x_j)$, $(x_i,x_l)$ and $(x_j,x_l)$ is identically zero.\\
\indent
It is also worthy to point out that in the generic situation described above, the number $s(n-s)-1$ in Lemma \ref{number of constraints} and hence the number ${n-1 \choose 2}$ of algebraically independent constraints in Proposition \ref{binomial growth} cannot be reduced. In other words, there are  functions for which all ratios $\frac{F_{x_ix_j}}{F_{x_i}F_{x_j}}\, (1\leq i\leq s, s+1\leq j\leq n)$ are the same except  
$\frac{F_{x_1x_n}}{F_{x_1}F_{x_n}}$; e.g. 
$$
F(x_1,\dots,x_n)=x_1x_n+\sum_{t=1}^{n-1}x_t 
$$
whose mixed partial derivatives $F_{x_ix_j}$ all vanish except $F_{x_1x_n}$.
\end{remark}

\subsection{Polynomial function setting}\label{polynomial setting}
This brief section is devoted to tree representations of polynomials. We are going to see that the local representation that Theorem \ref{main} suggests for an $n$-variate polynomial persists throughout $\Bbb{R}^n$ and one can always avoid usage of transcendental functions in the representation. 

\begin{proposition}\label{polynomial-global}
Let $P=P(x_1,\dots,x_n)$ be a polynomial satisfying the constraints \eqref{temp2} in Theorem \ref{main} for a binary tree $T$ with $n$ terminals. Then $P$ has a representation on $T$ that holds on the entirety of $\Bbb{R}^n$. Moreover, if $P$ is not independent of any of the variables $x_1,\dots,x_n$, the functions appeared in a representation of $P$ on the tree $T$ must be polynomial as well.
\end{proposition}

\begin{proof}
We first consider a tree representations of a polynomial $P$ not independent of any of its variables locally around a point, say $\mathbf{0}\in\Bbb{R}^n$. We are going to show that the functions appearing in this superposition must be polynomials as well. In the inductive constructions of Proposition \ref{presentation} or that of Section \ref{the proof} (with $P$ in place of $F$) always one of the functions to which the induction hypothesis is applied is obtained by setting some the coordinates to be zero, e.g. $\xi=P(x,y,0)$ or $\xi=P(x_1,\dots,x_s;0,\dots,0)$ (check \eqref{auxiliary 7} or \eqref{auxiliary 1}) which is a polynomial too. The other functions occurring in the construction, e.g. $g(w,z)$ from the proof of Proposition \ref{presentation} or the function appeared in \eqref{auxiliary 4}, are polynomial as well. The key point is if one locally presents a polynomial as a power series in terms of other polynomials which are non-constant, then the power series must terminate after finitely many terms. This could be easily deduced from the uniqueness of Taylor series. In particular, in \eqref{auxiliary 3}, the single variable functions $u_{k_{s+1},\dots,k_n}$ must be a polynomial as both the left hand side and $\xi$ are polynomials of $x_1,\dots,x_n$. Then in  \eqref{auxiliary 4} the function $g(\xi;x_{s+1},\dots,x_n)$ on a lesser number of variables would be a polynomial as well and so, arguing inductively, its presentation would be entirely in terms of polynomials.\\
\indent
Next, arguing globally, suppose $P=P(x_1,\dots,x_n)$ satisfies the constraints \eqref{temp2}. If $P$ is independent of one the variables $x_i$, then the problem reduces to representing polynomials on lesser number of variables by $T$. Thus let us assume $\frac{\partial P}{\partial x_i}\not\equiv 0$ for all $1\leq i\leq n$. Then there is a point of $\Bbb{R}^n$ at which all partial derivatives are non-zero. By the preceding local discussion, there is a tree representation of $P$ around that point which is necessarily in terms of polynomials. In particular; $P\in\mathcal{F}(T)$, and the representation remains valid on the entirety of $\Bbb{R}^n$ since globally defined analytic functions agreeing over a non-vacuous open set must coincide globally.    
\end{proof}

It is suggestive to abstractly define the space of polynomials subjected to constraints originating from a binary tree. To avoid infinite-dimensional spaces, we put a bound on degrees. 

\begin{definition}
The $k^{\rm{th}}$ \textit{tree variety} ${\rm{Var}}_{T}^k$ associated with the binary tree $T$ of $n$ leaves is the Zariski closed subset of the affine variety of $n$-variate (real or complex) polynomials consisting of polynomials $P(x_1,\dots,x_n)$ of total degree at most $k$ that satisfy 
$$
\frac{\partial^2 P}{\partial x_i\partial x_l}.\frac{\partial P}{\partial x_j} = \frac{\partial^2 P}{\partial x_j\partial x_l}.\frac{\partial P}{\partial x_i}.
$$
for any triple $(x_i,x_j,x_l)$ of leaves of $T$ in which $x_l$ is an outsider.
\end{definition}
These real or complex varieties could be interesting to study from the algebro-geometric perspective. (Consult \cite[chap. 1]{MR0463157} for the basic notions of algebraic geometry.) They could be thought of as subvarieties of the space ${\rm{Poly}}_{n}^k$ of polynomials on $n$ variables $x_1,\dots,x_n$ whose total degree does not exceed  $k$. This ambient space is an affine variety of dimension $\binom{k+n}{n}$. Here we find an upper bound on the dimension of the subvariety ${\rm{Var}}_{T}^k$. 

\begin{proposition}\label{dimension bound}
Let $T$ be a binary tree with $n$ leaves and $k$ a positive integer. The variety ${\rm{Var}}_{T}^k$ is of dimension at most $n(k^2+1)$.
\end{proposition}

\begin{proof}
We shall prove this as usual by induction on the number of leaves $n$. For $n=1$ or $n=2$ every polynomial of one or two variables and with total degree at most $k$ could be realized as a tree function. Therefore $\dim{\rm{Var}}_{T}^k=k+1$ and 
$\dim{\rm{Var}}_{T}^k=\binom{k+2}{2}$ when $T$ has one or two leaves respectively; and clearly $k+1\leq k^2+1$ and $\binom{k+2}{2}\leq 2k^2+2$ for all $k\geq 0$. \\
\indent
For the inductive step, following the notation of Theorem \ref{main}, we denote the left and right sub-trees by $T_1$ and $T_2$ which respectively have $x_1,\dots,x_s$ and $x_{s+1},\dots,x_n$ as leaves where $1\leq s<n$. So the vector 
$$
\mathbf{x}=(x_1,\dots,x_s;x_{s+1},\dots,x_n)
$$
of coordinates could be written as $\mathbf{x}=\left(\tilde{\mathbf{x}},\tilde{\tilde{\mathbf{x}}}\right)$
where
$$
\tilde{\mathbf{x}}:=(x_1,\dots,x_s)\quad\tilde{\tilde{\mathbf{x}}}:=(x_{s+1},\dots,x_n).
$$
Hence a polynomial $P(\mathbf{x})=P(x_1,\dots,x_n)\in{\rm{Var}}_{T}^k$ may be written as
\begin{equation}\label{auxiliary 10}
P(x_1,\dots,x_s;x_{s+1},\dots,x_n)=g\left(P_1(x_1,\dots,x_s),P_2(x_{s+1},\dots,x_n)\right)
\end{equation}
where $P_1\left(\tilde{\mathbf{x}}\right)=P_1(x_1,\dots,x_s)$ and $P_2\left(\tilde{\tilde{\mathbf{x}}}\right)=P_2(x_{s+1},\dots,x_n)$ admit representations on trees $T_1$ and $T_2$ respectively, and $g=g(u,v)$ is a polynomial in two variables. As $\deg P\leq k$, the total degree of $P_1$ or $P_2$ cannot be more than $k$ unless $g$ is independent of $u$ or $v$ in which case one could take $P_1$ or $P_2$ to be constant as well. So it is safe to assume $P_1\in{\rm{Var}}_{T_1}^k$ and $P_2\in{\rm{Var}}_{T_2}^k$. We condition on the degree of the polynomial $g=g(u,v)$ appeared in \eqref{auxiliary 10} with respect to its indeterminates as follows:  
\begin{itemize}
\item Suppose $g(u,v)$ is of degree at most one with respect to each indeterminate. Hence $g(u,v)=\alpha u+\beta v+\gamma$ for appropriate scalars $\alpha,\beta,\gamma$ and \eqref{auxiliary 10} could be rewritten as 
$$
P(x_1,\dots,x_s;x_{s+1},\dots,x_n)=\alpha.P_1(x_1,\dots,x_s)+\beta.P_2(x_{s+1},\dots,x_n)+\gamma.
$$
Clearly every polynomial space ${\rm{Var}}_{T',k'}$ is invariant under affine transformations because one can modify the bivariate polynomial at the root through composition by an affine transformation. Hence the preceding equality exhibits $P(x_1,\dots,x_s;x_{s+1},\dots,x_n)$ as the sum of 
$$\alpha.P_1(x_1,\dots,x_s)\in{\rm{Var}}_{T_1}^k$$
and 
$$\beta.P_2(x_{s+1},\dots,x_n)+\gamma\in{\rm{Var}}_{T_2}^k.$$
This amounts to an injective morphism $\oplus:{\rm{Var}}_{T_1}^k\times{\rm{Var}}_{T_2}^k\rightarrow{\rm{Var}}_{T}^k$ defined by addition. The dimension of the range is $\dim{\rm{Var}}_{T_1}^k+\dim{\rm{Var}}_{T_2}^k$ which by the induction hypothesis does not exceed
$$
s(k^2+1)+(n-s)(k^2+1)=n(k^2+1).
$$
\item Suppose either $\deg_ug$ or $\deg_vg$ is one, say the former. We are going to argue that the situation could again be reduced to the case of affine $g$ and hence the dimension of the corresponding locus is not greater than the dimension calculated above. As $g(u,v)$  is affine with respect to $u$, by the same argument as before one can absorb that affine part and rewrite \eqref{auxiliary 10} as
\begin{equation}\label{auxiliary 11}
P(x_1,\dots,x_s;x_{s+1},\dots,x_n)=P'_1(x_1,\dots,x_s)+h\left(P_2(x_{s+1},\dots,x_n)\right)
\end{equation}
where $P'_1\in{\rm{Var}}_{T_1}^k$ is an appropriate affine transform $\alpha.P_1+\beta$ of $P_1$, $h(v):=g(0,v)$ and $P_2$ is a tree polynomial for $T_2$. Applying a single variable polynomial to the output of the root of course results in another tree polynomial; i.e. $P'_2:=h\left(P_2(x_{s+1},\dots,x_n)\right)$ -- whose degree is at most $\deg P\leq k$ due to \eqref{auxiliary 11} -- belongs to ${\rm{Var}}_{T_2}^k$ too. Therefore 
$$P(x_1,\dots,x_s;x_{s+1},\dots,x_n)=P'_1(x_1,\dots,x_s)+P'_2(x_{s+1},\dots,x_n)$$ 
is a sum of elements of ${\rm{Var}}_{T_1}^k$ and ${\rm{Var}}_{T_2}^k$. So we simply could revert back to the case that we have already studied.
\item Finally, suppose $\deg_ug$ and $\deg_vg$ are both larger than one. Just like the previous part, we could assume that neither $P_1$ nor $P_2$ is constant.  Now in \eqref{auxiliary 10}, the degrees of $P_1$ and $P_2$ do not exceed $\frac{k}{\deg_ug}$ and $\frac{k}{\deg_vg}$ respectively as otherwise on the left hand side $\deg P$ would be larger than $k$. Therefore, denoting $\deg P_1$ and $\deg P_2$ by $j_1$ and $j_2$, we should have 
\begin{equation}\label{range1}
1\leq j_1,j_2\leq\lfloor\frac{k}{2}\rfloor.
\end{equation}
Notice that this implies $k\geq 2$. Now the polynomial $g(u,v)$ can only have monomials $u^cv^d$ with $cj_1+dj_2\leq k$ as otherwise the degree of $g(P_1,P_2)$ would be larger than $k$. Hence $g$ must belong to the following linear space of polynomials
\begin{equation}\label{range2}
{\rm{V}}_{j_1,j_2}={\rm{Span}}\left\{u^cv^d\mid cj_1+dj_2\leq k\right\} 
\end{equation}
whose dimension is not larger than the dimension $\binom{k+2}{2}$ of the space of all bivariate polynomials of degree at most $k$.
We thus need to consider the ranges of morphisms 
\small
\begin{equation}\label{range3}
\begin{cases}
{\rm{V}}_{j_1,j_2}\times{\rm{Var}}_{T_1}^{j_1}\times{\rm{Var}}_{T_2}^{j_2}\rightarrow{\rm{Var}}_{T}^k\\
\left(g(u,v),P_1(x_1,\dots,x_s),P_2(x_{s+1},\dots,x_n)\right)\mapsto g\left(P_1(x_1,\dots,x_s),P_2(x_{s+1},\dots,x_n)\right) 
\end{cases}
\end{equation}
\normalsize
for $j_1,j_2$ such as in \eqref{range1} and the variety ${\rm{V}}_{j_1,j_2}$ as defined in \eqref{range2}. We need to argue that the dimension 
$$
\dim{\rm{V}}_{j_1,j_2}+\dim{\rm{Var}}_{T_1}^{j_1}+\dim{\rm{Var}}_{T_2}^{j_2} 
$$
of the domain of \eqref{range3} is not greater than $n(k^2+1)$. According to the induction hypothesis:
\small
$$
\dim{\rm{Var}}_{T_1}^{j_1}\leq s\left(\left(\frac{k}{2}\right)^2+1\right),\quad \dim{\rm{Var}}_{T_2}^{j_2}\leq (n-s)\left(\left(\frac{k}{2}\right)^2+1\right).
$$
\normalsize
We conclude that:
\small
\begin{equation*}
\begin{split}
&\dim{\rm{V}}_{j_1,j_2}+\dim{\rm{Var}}_{T_1}^{j_1}+\dim{\rm{Var}}_{T_2}^{j_2}\\
&\leq\binom{k+2}{2}+s\left(\left(\frac{k}{2}\right)^2+1\right)+(n-s)\left(\left(\frac{k}{2}\right)^2+1\right)\\
&=\left[\binom{k+2}{2}-\frac{3k^2}{2}\right]+2(k^2+1) +(s-1)\left(\left(\frac{k}{2}\right)^2+1\right)+(n-s-1)\left(\left(\frac{k}{2}\right)^2+1\right)\\
&\leq\left[\binom{k+2}{2}-\frac{3k^2}{2}\right]+2(k^2+1)+(s-1)(k^2+1)+(n-s-1)(k^2+1)\\
&=\left[\binom{k+2}{2}-\frac{3k^2}{2}\right]+n(k^2+1)\leq n(k^2+1);
\end{split}    
\end{equation*}
\normalsize
where for the last step we have used the fact that $\frac{3k^2}{2}\geq\binom{k+2}{2}$ whenever $k\geq 2$.
\end{itemize}
\end{proof}
\begin{question}
Is there a formula for $\dim{\rm{Var}}_{T}^k$ purely in terms of $k$ and the number of leaves of the binary tree $T$?
\end{question}

\section{Bit-valued function setting}\label{discrete setting}
In this section we investigate the case of bit-valued functions. As outlined in Definition \ref{definition-functions}, the inputs and the output are from $\{0,1\}$; hence we are dealing with functions of the form  $\{0,1\}^n\rightarrow\{0,1\}$ where $n$ is the number of terminals of the tree $T$ under consideration. There are $2^{2^n}$ such functions whereas the number of those which could be implemented on a tree turns out to be far less than this super-exponential number. Theorem \ref{discrete} below and the succeeding corollary provide us with an explicit formula for the number $\big|\mathcal{F}_{\rm{bin}}(T)\big|$ of such functions which turns out to be only exponential in $n$.

\subsection{Discrete tree functions}
\label{discrete-tree}
Before proceeding with enumerating tree functions, it is essential to mention that each elements of $\mathcal{F}_{\rm{bin}}(T)$ is basically a polynomial in $\Bbb{Z}_2\left[x_1,\dots,x_n\right]$. This is due to the fact that every function $g:\{0,1\}^2\rightarrow\{0,1\}$ assigned to a node can be realized as a bivariate polynomial over the field of two elements $\Bbb{Z}_2$:
\begin{equation}\label{two variable}
\begin{split}
g(x,y)=&\left(g(1,1)+g(1,0)+g(0,1)+g(0,0)\right)xy+\left(g(1,0)+g(0,0)\right)x\\
&+\left(g(0,1)+g(0,0)\right)y+g(0,0).
\end{split}
\end{equation}
\normalsize
We are going to return to this point of view later in this subsection where we formulate binary analogous of the constraint \eqref{temp2} using formal differentiation of polynomials in $\Bbb{Z}_2\left[x_1,\dots,x_n\right]$.

The following theorem provides us with a recursive construction of binary tree functions. 
\begin{theorem}\label{discrete}
Let $T_1$, $T_2$ be binary trees with $n_1$ and $n_2$ terminals respectively. Connecting them through a node results in a tree $T$ with that node as its root. Labeling the leaves of $T_1$ as 
$\mathbf{x}=\left(x_1,\dots,x_{n_1}\right)$ and the leaves of $T_2$ as $\mathbf{y}=\left(y_1,\dots,y_{n_2}\right)$, the space $\mathcal{F}_{\rm{bin}}(T)$ of functions
$$
F(\mathbf{x},\mathbf{y})=F(x_1,\dots,x_{n_1};y_1,\dots,y_{n_2})
$$
represented by tree $T$ can be described in terms of smaller spaces $\mathcal{F}_{\rm{bin}}(T_1)$ and $\mathcal{F}_{\rm{bin}}(T_2)$ as the disjoint union below:
\\
\small
\begin{equation}\label{description}
\begin{split}
\mathcal{F}_{\rm{bin}}(T)=&\left\{F_1(\mathbf{x})F_2(\mathbf{y})\big| F_1(\mathbf{x})\in\mathcal{F}_{\rm{bin}}(T_1) \text{ and } F_2(\mathbf{y})\in\mathcal{F}_{\rm{bin}}(T_2)\text{ non-constant}\right\}\\
&\bigsqcup\left\{F_1(\mathbf{x})F_2(\mathbf{y})+1\big| F_1(\mathbf{x})\in\mathcal{F}_{\rm{bin}}(T_1) \text{ and } F_2(\mathbf{y})\in\mathcal{F}_{\rm{bin}}(T_2)\text{ non-constant}\right\}\\
&\bigsqcup\left\{F_1(\mathbf{x})+F_2(\mathbf{y})\big| F_1(\mathbf{x})\in\mathcal{F}_{\rm{bin}}(T_1)\text{ and } F_2(\mathbf{y})\in\mathcal{F}_{\rm{bin}}(T_2)\text{ non-constant, }F_1(\mathbf{0})=0 \right\}\\
&\bigsqcup\left\{F_1(\mathbf{x})\big| F_1(\mathbf{x})\in\mathcal{F}_{\rm{bin}}(T_1)\text{ non-constant}\right\}\\
&\bigsqcup\left\{F_2(\mathbf{y})\big| F_2(\mathbf{y})\in\mathcal{F}_{\rm{bin}}(T_2)\text{ non-constant}\right\}\\
&\bigsqcup\left\{\text{constant functions } F\equiv 0,1\right\}.
\end{split}    
\end{equation}
\normalsize
\end{theorem}

\begin{proof}
Removing the root of $T$ leaves us with two rooted sub-trees $T_1,T_2$. The function $g:\{0,1\}^2\rightarrow\{0,1\}$ at the root then takes in the outputs of the functions $F_1=F_1(\mathbf{x})$ and $F_2=F_2(\mathbf{y})$ which are implemented on $T_1$ and $T_2$ respectively. By \eqref{two variable} any such a function can be realized as a bivariate polynomial over $\Bbb{Z}_2$. Therefore the right hand side of \eqref{description} is indeed a subset of $\mathcal{F}_{\rm{bin}}(T)$. We need to show that all functions from $\mathcal{F}_{\rm{bin}}(T)$ have appeared there and moreover, the union is disjoint. We categorize polynomials $g\in\Bbb{Z}_2[x,y]$ whose degree w.r.t. each indeterminate is at most one as follows:
\begin{enumerate}[(i)]
\item $xy, xy+x=x(y+1), xy+y=(x+1)y, xy+x+y+1=(x+1)(y+1);$
\item $xy+1, xy+x+1=x(y+1)+1, xy+y+1=(x+1)y+1, xy+x+y=(x+1)(y+1)-1;$
\item $x+y,x+y+1=(x+1)+y;$
\item $x,x+1;$
\item $y,y+1;$
\item $0;$
\item $1.$
\end{enumerate}
These polynomials give rise to all 16 possibilities for $\{0,1\}^2\rightarrow\{0,1\}$; cf. \eqref{two variable}. For any $F_i\in\mathcal{F}_{\rm{bin}}(T_i) (i\in\{1,2\})$, the function $F_i+1=F_i-1$ obviously belongs to $\mathcal{F}_{\rm{bin}}(T_i)$ as well. Consequently, for the operation to take place at the root it suffices to consider only one representative from each group as other choices give rise to the same function 
$$
F(\mathbf{x},\mathbf{y}):=g\left(F_1(\mathbf{x})),F_2(\mathbf{y})\right)
$$
for slightly different inputs $F_1$ and $F_2$ from $\mathcal{F}_{\rm{bin}}(T_1)$ and $\mathcal{F}_{\rm{bin}}(T_2)$. Going with the first polynomial of each group as $g$, we conclude that the expressions below yield all functions in $\mathcal{F}_{\rm{bin}}(T)$ as $F_1$ and $F_2$ vary in $\mathcal{F}_{\rm{bin}}(T_1)$ and $\mathcal{F}_{\rm{bin}}(T_2)$ respectively:
$$
F_1(\mathbf{x})F_2(\mathbf{y}), F_1(\mathbf{x})F_2(\mathbf{y})+1, F_1(\mathbf{x})+F_2(\mathbf{y}), F_1(\mathbf{x}), F_2(\mathbf{y}), \text{constant functions}.
$$
Of course, in the first three if either $F_1$ or $F_2$ is constant, the expression would be in the form of one of the last three. Moreover, in the third expression there, changing $F_1(\mathbf{x})+F_2(\mathbf{y})$ to $\left(F_1(\mathbf{x})+1\right)+\left(F_2(\mathbf{y})+1\right)$ if necessary, it is safe to assume that the first function is zero at the origin.
Therefore, the union on the right of \eqref{description} is the whole $\mathcal{F}_{\rm{bin}}(T)$. The last step is to show that the subsets in this union are disjoint. Let 
$(F_1,F_2)\in\mathcal{F}_{\rm{bin}}(T_1)\times\mathcal{F}_{\rm{bin}}(T_2)$ be a pair of non-constant functions. 
\begin{enumerate}[(a)]
\item The functions $F_1(\mathbf{x})F_2(\mathbf{y})$, $F_1(\mathbf{x})F_2(\mathbf{y})+1$ and $F_1(\mathbf{x})+F_2(\mathbf{y})$ are all non-constant: If otherwise, plugging a $\mathbf{y}_0$ with $F_2(\mathbf{y}_0)=1$ implies that $F_1(\mathbf{x})$ is constant; a contradiction.
\end{enumerate}
Next, let $(\tilde{F}_1,\tilde{F}_2)$ be another pair of non-constant functions picked from 
$\mathcal{F}_{\rm{bin}}(T_1)\times\mathcal{F}_{\rm{bin}}(T_2)$ as well.  
\begin{enumerate}[(a)]
\setcounter{enumi}{1}
\item The functions $F_1(\mathbf{x})F_2(\mathbf{y})$, $\tilde{F}_1(\mathbf{x})$ and $\tilde{F}_2(\mathbf{y})$ are different: Plugging a $\mathbf{y}_0$ with $F_2(\mathbf{y}_0)=1$ makes the third one a constant function while the first two are still non-constant. Similarly, plugging an $\mathbf{x}_0$ with $F_1(\mathbf{x}_0)=1$ implies $F_1(\mathbf{x})F_2(\mathbf{y})\not\equiv\tilde{F}_1(\mathbf{x})$.
\item $F_1(\mathbf{x})+F_2(\mathbf{y})\not\equiv\tilde{F}_1(\mathbf{x})$ and $F_1(\mathbf{x})+F_2(\mathbf{y})\not\equiv \tilde{F}_2(\mathbf{y})$ as otherwise, the function $F_1(\mathbf{x})$ would be independent of $\mathbf{x}$ and hence constant; a contradiction.
\item $F_1(\mathbf{x})F_2(\mathbf{y})\not\equiv\tilde{F}_1(\mathbf{x})+\tilde{F}_2(\mathbf{y})$ and $F_1(\mathbf{x})F_2(\mathbf{y})+1\not\equiv\tilde{F}_1(\mathbf{x})+\tilde{F}_2(\mathbf{y})$: Plugging a $\mathbf{y}_0$ with $F_2(\mathbf{y}_0)=0$, the left hand side would be zero or one whereas the right hand side is either $\tilde{F}_1(\mathbf{x})$ or $\tilde{F}_1(\mathbf{x})+1$, both of them non-constant.
\end{enumerate}

Assuming that furthermore $(F_1,F_2)\neq(\tilde{F}_1,\tilde{F}_2)$:
\begin{enumerate}[(a)]
\setcounter{enumi}{4}
\item We claim that it is impossible for $F_1(\mathbf{x})F_2(\mathbf{y})-\tilde{F}_1(\mathbf{x})\tilde{F}_2(\mathbf{y})$ to be a constant function. Assume the contrary. Again, evaluate at a $\mathbf{y}_0$ with $F_2(\mathbf{y}_0)=1$. This means either $F_1(\mathbf{x})$ or $F_1(\mathbf{x})-\tilde{F}_1(\mathbf{x})$ should be constant based on whether $\tilde{F}_2(\mathbf{y}_0)$ is zero or one. The former is impossible and the latter implies $F_1(\mathbf{x})=\tilde{F}_1(\mathbf{x})+\epsilon_1$ where $\epsilon_1\in\{0,1\}$. Repeating this argument with an $\mathbf{x}_0$ where $F_1(\mathbf{x}_0)=1$ requires the other two polynomials to satisfy a similar relation: $F_2(\mathbf{y})=\tilde{F}_2(\mathbf{x})+\epsilon_2$. But then
\begin{equation*}
\begin{split}
&F_1(\mathbf{x})F_2(\mathbf{y})-\tilde{F}_1(\mathbf{x})\tilde{F}_2(\mathbf{y})=\left(\tilde{F}_1(\mathbf{x})+\epsilon_1\right)\left(\tilde{F}_2(\mathbf{y})+\epsilon_2\right)-\tilde{F}_1(\mathbf{x})\tilde{F}_2(\mathbf{y})\\
&=\epsilon_1\tilde{F}_2(\mathbf{x})+\epsilon_2\tilde{F}_1(\mathbf{y})+\epsilon_1\epsilon_2
\end{split}
\end{equation*}
must be constant. The case of $\epsilon_1=\epsilon_2=1$ has been ruled out before in (a); the case of $\epsilon_1=\epsilon_2=0$ simply means $(F_1,F_2)=(\tilde{F}_1,\tilde{F}_2)$ which is impossible; and finally, if only one of $\epsilon_1,\epsilon_2$ is one, then either $\tilde{F}_2(\mathbf{x})$ or $\tilde{F}_1(\mathbf{y})$ must be constant as well. So the aforementioned claim is proven.  
\end{enumerate}
So far, we have established the disjointness of the subsets appeared in the union \eqref{description}. Nevertheless, we continue with one more observation that comes in handy soon in Corollary \ref{formula}.
Let $(F_1,F_2),(\tilde{F}_1,\tilde{F}_2)$ be as in (e) with the additional property of $F_1(\mathbf{0})=\tilde{F}_1(\mathbf{0})=0$.
\begin{enumerate}[(a)]
\setcounter{enumi}{5}
\item $F_1(\mathbf{x})+F_2(\mathbf{y})\not\equiv\tilde{F}_1(\mathbf{x})+\tilde{F}_2(\mathbf{y})$ as otherwise $F_1(\mathbf{x})-\tilde{F}_1(\mathbf{x})$ and $\tilde{F}_2(\mathbf{y})-F_2(\mathbf{y})$ coincide and hence are both constant functions of the value $F_1(\mathbf{0})-\tilde{F}_1(\mathbf{0})=0$; that is, $F_1=\tilde{F}_1$ and $F_2=\tilde{F}_2$; a contradiction. 
\end{enumerate}

\end{proof}

\begin{corollary}\label{formula}
For a binary tree with $n$ terminals the number of functions $\{0,1\}^n\rightarrow\{0,1\}$ that could be implemented on $T$ is given by:
\begin{equation}\label{binary function formula}
\left|\mathcal{F}_{\rm{bin}}(T)\right|= \frac{2\times 6^n+8}{5}.    
\end{equation}
\end{corollary}

\begin{proof}
\
The formula clearly works for the base cases $n=1,2$ where one gets $4,16$; i.e. the number of functions $\{0,1\}\rightarrow\{0,1\}$ or $\{0,1\}^2\rightarrow\{0,1\}$. To do the inductive step, suppose the formula holds for trees $T_1,T_2$ in Theorem \ref{discrete}. In the description of $\mathcal{F}_{\rm{bin}}(T)$ as a disjoint union in \eqref{description}, the cardinality of each subset can be easily calculated: For the first three subsets, observations (e) and (f) indicate that different pairs $\left(F_1=F_1(\mathbf{x}),F_2=F_2(\mathbf{y})\right)$ result in different functions $F=F(\mathbf{x},\mathbf{y})$. Therefore, both of the first two sets are of cardinality $\left(\big|\mathcal{F}_{\rm{bin}}(T_1)\big|-2\right).\left(\big|\mathcal{F}_{\rm{bin}}(T_2)\big|-2\right)$ where subtracting $2$ is for excluding constant functions. As for the third set on the right hand side of \eqref{description}, we need to divide by two as well since half of the functions in $\mathcal{F}_{\rm{bin}}(T_1)$ are zero at $\mathbf{0}$ and the other half are one. Adding up the sizes of the subsets appeared in this partition:
\small
\begin{equation*}
\begin{split}
\big|\mathcal{F}_{\rm{bin}}(T)\big|&=\left(\big|\mathcal{F}_{\rm{bin}}(T_1)\big|-2\right).\left(\big|\mathcal{F}_{\rm{bin}}(T_2)\big|-2\right)
+\left(\big|\mathcal{F}_{\rm{bin}}(T_1)\big|-2\right).\left(\big|\mathcal{F}_{\rm{bin}}(T_2)\big|-2\right)\\
&+\left(\frac{\big|\mathcal{F}_{\rm{bin}}(T_1)\big|-2}{2}\right).\left(\big|\mathcal{F}_{\rm{bin}}(T_2)\big|-2\right)+\left(\big|\mathcal{F}_{\rm{bin}}(T_1)\big|-2\right)+\left(\big|\mathcal{F}_{\rm{bin}}(T_2)\big|-2\right)\\
&+2=\frac{5}{2}\left(\big|\mathcal{F}_{\rm{bin}}(T_1)\big|-2\right).\left(\big|\mathcal{F}_{\rm{bin}}(T_2)\big|-2\right)+\big|\mathcal{F}_{\rm{bin}}(T_1)\big|+\big|\mathcal{F}_{\rm{bin}}(T_2)\big|-2.
\end{split}    
\end{equation*}
\normalsize
Substituting $\big|\mathcal{F}_{\rm{bin}}(T_1)\big|$ and $\big|\mathcal{F}_{\rm{bin}}(T_2)\big|$ from the induction hypothesis: 
\small
\begin{equation*}
\begin{split}
\big|\mathcal{F}_{\rm{bin}}(T)\big|&=\frac{5}{2}.\frac{2\times 6^{n_1}-2}{5}.\frac{2\times 6^{n_2}-2}{5}+\frac{2\times 6^{n_1}+8}{5}+\frac{2\times 6^{n_2}+8}{5}-2\\
&=\frac{\left(2\times 6^{n_1}-2\right).\left(6^{n_2}-1\right)}{5}+\frac{2\times 6^{n_1}+2\times 6^{n_2}+6}{5}\\
&=\frac{2\times 6^{n_1+n_2}+8}{5};
\end{split}
\end{equation*}
\normalsize
that is, \eqref{binary function formula} for the tree $T$ which has $n_1+n_2$ leaves.
\end{proof}
This result once again attests to a recurring theme of this paper: the subset of tree functions is very small compared to whole space of functions. For instance, for $n=4$ we have $2^{16}=65536$ functions but only $520$ of them are expressible in the trees. 
\begin{corollary}\label{sparse}
The size of the set of functions $\{0,1\}^n\rightarrow\{0,1\}$ that are tree functions for a labeled binary rooted tree with $n$ leaves is $O\left(24^n\right)$ and is thus $o\left(2^{2^n}\right)$.
\end{corollary}
\begin{proof}
The number of bit-valued functions corresponding to any binary rooted tree with $n$ leaves is $O\left(6^n\right)$ by Corollary \ref{formula}. The number of such trees with leaves labeled by $x_1,\dots,x_n$ is clearly $O\left(4^n\right)$: It is not hard to see that as a graph it has $2n-1$ vertices and on the other hand, the number of graphs whose vertices are prescribed as a set of size $2n-1$ is $2^{2n-1}$.   
\end{proof}

The rest of this section is devoted to discrete versions of the constraints we have been working with in previous subsections and their necessity and sufficiency. The ambient space $D_n$ of functions $\{0,1\}^n\rightarrow\{0,1\}$ in \eqref{binary functions} can be identified with the following space of polynomials 
\begin{equation}\label{polynomial space}
\left\{\sum_{S\subseteq \{1,\dots,n\}}\epsilon(S)\prod_{s\in S}x_s\,\big|\epsilon(S)\in\{0,1\}\right\}; 
\end{equation}
that is, polynomials with binary coefficients on $n$ indeterminates whose degree w.r.t. every indeterminate is not larger than one.
Notice that \eqref{temp2} again furnishes us with a necessary condition for such a polynomial to be represented on a tree $T$ as one can differentiate formally in the ring $\Bbb{Z}_2\left[x_1,\dots,x_n\right]$: for any triple of leaves $x_i,x_j,x_l$ as in Theorem \ref{main}, the identity
\begin{equation}\label{discrete differentiation}
\frac{\partial^2P}{\partial x_i\partial x_l}.\frac{\partial P}{\partial x_j}=\frac{\partial^2P}{\partial x_j\partial x_l}.\frac{\partial P}{\partial x_i}    
\end{equation}
must hold.
Hence \eqref{non-example} again serves as a non-example; that is, a binary function of three variables with no tree presentation. 

\begin{remark}\label{technical}
A diligent reader with mathematical background may find \eqref{discrete differentiation} and the derivatives therein problematic since we are moving back and forth between polynomials with coefficients in $\Bbb{Z}_2$ and functions with binary inputs; and over a finite field one can easily find pair of polynomials whose values at every single point coincide without being the same in the polynomial sense (which is to have the same coefficients). To address this concern, recall that the space $D_n$ of functions $\{0,1\}^n\rightarrow\{0,1\}$ has been identified with the space of $n$-variate polynomials in \eqref{polynomial space} via thinking of polynomials as functions $\{0,1\}^n\rightarrow\{0,1\}$ by evaluating them at binary vectors of length $n$, and this is an one-to-one correspondence. It is not hard to see that given polynomials $P_1(x_1,\dots,x_s)$ and $P_2(x_{s+1},\dots,x_n)$ of this form and a function $g:\{0,1\}^2\rightarrow\{0,1\}$ -- which by \eqref{two variable} can always be realized as a polynomial of the same from but on two indeterminates  -- the $n$-variate polynomial
$$
P(x_1,\dots,x_n)=g\left(P_1(x_1,\dots,x_s),P_2(x_{s+1},\dots,x_n)\right)
$$
is also of the form appeared in \eqref{polynomial space}. That is to say, these families of polynomials are closed under tree superpositions. We conclude that the presentation of a binary function $P\in\mathcal{F}_{\rm{bin}}(T)$ as a composition of bivariate functions $\{0,1\}^2\rightarrow\{0,1\}$ is indeed a bona fide presentation of an $n$-variate polynomial in terms of bivariate polynomials since once two polynomials from \eqref{polynomial space} give rise to same functions, they coincide as polynomials too. It makes sense to formally take derivatives when we have an equality of polynomials over $\Bbb{Z}_2$, hence \eqref{discrete differentiation} holds for any triple of variables $x_i,x_j,x_l$ with $x_l$ the outsider of the triple.
\end{remark}

\begin{remark}
There is furthermore a discrete interpretation of differentiation in this context: Given an $n$-variate polynomial
$$
P(\mathbf{x})=P(x_1,\dots,x_n)
$$
from \eqref{polynomial space}, $\frac{\partial F}{\partial x_i}(\mathbf{x})$ is the same as 
$$P(x+\mathbf{e}_i)-P(\mathbf{x})=P(x_1,\dots,x_{i-1},x_i+1,x_{i+1},\dots,x_n)-P(x_1,\dots,x_n)$$ 
where $\mathbf{e}_i$ is the vector from the standard basis whose $i^{\rm{th}}$ entry is $1$ and the rest are $0$; and of course there is no difference between addition and subtraction modulo two. Consequently, we arrive at the following binary version of the constraints \eqref{discrete differentiation}:  
\small
\begin{equation}\label{discrete constraint}
\begin{split}
&\left[P(\mathbf{x}+\mathbf{e}_i+\mathbf{e}_l)+P(\mathbf{x}+\mathbf{e}_i)+P(\mathbf{x}+\mathbf{e}_l)+P(\mathbf{x})\right].\left[P(\mathbf{x}+\mathbf{e}_j)+P(\mathbf{x})\right]\\
&=\left[P(\mathbf{x}+\mathbf{e}_j+\mathbf{e}_l)+P(\mathbf{x}+\mathbf{e}_j)+P(\mathbf{x}+\mathbf{e}_l)+P(\mathbf{x})\right].\left[P(\mathbf{x}+\mathbf{e}_i)+P(\mathbf{x})\right];
\end{split}
\end{equation}
\normalsize
for any triple of variables $(x_i,x_j,x_l)$ in which $x_l$ is an outsider.
\end{remark}

We conclude the subsection by showing that the constraints on the partial derivatives are sufficient in the binary context too.  
\begin{theorem}\label{main binary}
Let $T$ be a rooted tree with $n$ leaves and $P(x_1,\dots,x_n)$ a polynomial in the form of \eqref{polynomial space}; i.e. a polynomial of $n$ variables over $\Bbb{Z}_2$ whose degree with respect to each $x_i$ is one. 
Thinking of $P$ as a function $\{0,1\}^n\rightarrow\{0,1\}$, it belongs to $\mathcal{F}_{\rm{bin}}(T)$ if and only if for any three leaves $x_i,x_j,x_l$ of $T$ with $x_i,x_j$ in a rooted sub-tree that does not have $x_l$ the identity \eqref{discrete differentiation} holds.
\end{theorem}

\begin{proof}
The necessity has been discussed before in this section and we focus on the sufficiency of constraints \eqref{discrete differentiation}. This is going to be achieved by the usual inductive argument. 
The base case  where $T$ has just one or two leaves is clear because every function can be realized as a polynomial (cf. \eqref{two variable}) and the condition \eqref{discrete differentiation} is automatic. 
For a general $T$, denote the sub-trees to the left and right of the root by $T_1$ and $T_2$. Suppose their leaves are labeled as 
$$
\tilde{\mathbf{x}}:=(x_1,\dots,x_s)\quad\tilde{\tilde{\mathbf{x}}}:=(x_{s+1},\dots,x_n).
$$
It is clear that for any two tree functions $P_i\in\mathcal{F}_{\rm{bin}}(T_i)$ ($i\in\{1,2\}$) the functions 
\begin{equation}\label{sum form}
(\tilde{\mathbf{x}},\tilde{\tilde{\mathbf{x}}})\mapsto P_1(\tilde{\mathbf{x}})+P_2(\tilde{\tilde{\mathbf{x}}}) 
\end{equation}
or
\begin{equation}\label{product form 1}
(\tilde{\mathbf{x}},\tilde{\tilde{\mathbf{x}}})\mapsto P_1(\tilde{\mathbf{x}})P_2(\tilde{\tilde{\mathbf{x}}}) 
\end{equation}
or
\begin{equation}\label{product form 2}
(\tilde{\mathbf{x}},\tilde{\tilde{\mathbf{x}}})\mapsto P_1(\tilde{\mathbf{x}})P_2(\tilde{\tilde{\mathbf{x}}})+1
\end{equation}
can be represented on the original tree $T$ as they amount to assigning  $(a,b)\mapsto a+b$, $(a,b)\mapsto ab$ or  $(a,b)\mapsto ab+1$ to the root. Thus it suffices to argue that conversely, every 
$P\in\mathcal{F}_{\rm{bin}}(T)$ has such a presentation for  polynomials $P_1=P_1(x_1,\dots,x_s)$ and $P_2=P_2(x_{s+1},\dots,x_n)$ of the same form but on a lesser number of indeterminates that belong to $\mathcal{F}_{\rm{bin}}(T_1)$ and $\mathcal{F}_{\rm{bin}}(T_2)$ respectively.
Write $P(x_1,\dots,x_n)$ as
\begin{equation}\label{auxiliary 8}
P(x_1,\dots,x_s;x_{s+1},\dots,x_n)=\sum_{U\subseteq\{s+1,\dots,n\}}P_{U}(x_1,\dots,x_s)\prod_{u\in U}x_u. 
\end{equation}
Invoking the generalization of \eqref{discrete differentiation} provided by Lemma \ref{only one constraint}, for any two $1\leq i,j\leq s$ and any non-empty subset $U_0\subseteq\{s+1,\dots,n\}$ of the second half of indices, we have 
$$
\frac{\partial^{1+|U_0|}P}{\partial x_i\prod_{u\in U_0}\partial x_u}.\frac{\partial P}{\partial x_j}
=\frac{\partial^{1+|U_0|}P}{\partial x_j\prod_{u\in U_0}\partial x_u}.\frac{\partial P}{\partial x_i};
$$
because $x_i,x_j$ are to the ``left'' of the node while $x_u$'s lie to the ``right''. In view of expression \eqref{auxiliary 8}, this implies
$$
\frac{\partial P_{U_0}}{\partial x_i}.\left(\sum_{U\subseteq\{s+1,\dots,n\}}\frac{\partial P_{U}}{\partial x_j}\prod_{u\in U}x_u\right)
=\frac{\partial P_{U_0}}{\partial x_j}.\left(\sum_{U\subseteq\{s+1,\dots,n\}}\frac{\partial P_{U}}{\partial x_i}\prod_{u\in U}x_u\right).
$$
(Keep in mind that $P_U$'s are dependent only on the first $s$ indeterminates.) Equating the coefficients of each $\prod_{u\in U}x_u$ in both sides results in:
$$
\frac{\partial P_{U_0}}{\partial x_i}.\frac{\partial P_{U}}{\partial x_j}=\frac{\partial P_{U_0}}{\partial x_j}.\frac{\partial P_{U}}{\partial x_i}\quad (\forall U\subseteq\{s+1,\dots,n\} \text{ and }\forall 1\leq i,j\leq s).
$$
This simply means that gradient vectors $\left[\frac{\partial P_U}{\partial x_i}\right]_{1\leq i\leq s}$ of polynomials $P_U=P_U(x_1,\dots,x_s)$ are mutually linearly dependent as 
$U$ varies among subsets of $\left\{s+1,\dots,n\right\}$. Consequently, any two of them differ solely by their constant terms. Taking the constant term of $P_U$'s out of the summation in \eqref{auxiliary 8} and factoring out, we rewrite the equality as 
\begin{equation}\label{auxiliary 9}
P(x_1,\dots,x_s;x_{s+1},\dots,x_n)=Q(x_1,\dots,x_s)R(x_{s+1},\dots,x_n)+\tilde{R}(x_{s+1},\dots,x_n);
\end{equation}
where $Q(\overbrace{0,\dots,0}^{s\text{ times}})=0$. If $Q\equiv 0 \text{ or } 1$, $P$ is dependent only on the last $n-s$ indeterminates and hence is already in the from of \eqref{sum form}. Otherwise, a high enough partial derivative of it would be $1$. Reversing the roles of indices larger than $s$ and those not greater than $s$, pick a subset $U'\subseteq\{1,\dots,s\}$ with 
$$
\frac{\partial^{|U'|}Q}{\prod_{u'\in U'}\partial x_{u'}}\equiv 1
$$
and indices $s<i',j'\leq n$. Plugging \eqref{auxiliary 9} in  
$$
\frac{\partial^{1+|U'|}P}{\partial x_{i'}\prod_{u'\in U'}\partial x_{u'}}.\frac{\partial P}{\partial x_{j'}}
=\frac{\partial^{1+|U'|}P}{\partial x_{j'}\prod_{u'\in U'}\partial x_{u'}}.\frac{\partial P}{\partial x_{i'}}
$$
then results in
$$
\frac{\partial R}{\partial x_{i'}}.\left(Q\frac{\partial R}{\partial x_{j'}}+\frac{\partial \tilde{R}}{\partial x_{j'}}\right)
=\frac{\partial R}{\partial x_{j'}}.\left(Q\frac{\partial R}{\partial x_{i'}}+\frac{\partial \tilde{R}}{\partial x_{i'}}\right).
$$
Simplifying this, we arrive at
$$
\frac{\partial R}{\partial x_{i'}}.\frac{\partial \tilde{R}}{\partial x_{j'}}=\frac{\partial R}{\partial x_{j'}}.\frac{\partial \tilde{R}}{\partial x_{i'}}
$$
for any two indices $i',j'>s$. By the same argument as before, we conclude that up to a constant, polynomials 
$R(x_{s+1},\dots,x_n)$ and $\tilde{R}(x_{s+1},\dots,x_n)$ are scalar multiples of each other. Hence either $R\equiv 0,1$ 
-- in which case $P$ would be the sum of a polynomial of $x_1,\dots,x_s$ and a polynomial of $x_{s+1},\dots,x_n$ and hence in the form of \eqref{sum form} -- 
or there are  $\tilde{\epsilon},\delta\in\{0,1\}$ for which 
$\tilde{R}=\tilde{\epsilon}R+\delta$. Substituting in \eqref{auxiliary 9}:
$$
P(x_1,\dots,x_s;x_{s+1},\dots,x_n)=\left(Q(x_1,\dots,x_s)+\tilde{\epsilon}\right)R(x_{s+1},\dots,x_n)+\delta.
$$
So we have a presentation such as \eqref{product form 1} when $\delta=0$ and a presentation such as \eqref{product form 2} when $\delta=1$.

The final step would be to argue that in either of the assignments \eqref{sum form}, \eqref{product form 1} or \eqref{product form 2} $P_1$ and $P_2$ could be chosen from $\mathcal{F}_{\rm{bin}}(T_1)$ and $\mathcal{F}_{\rm{bin}}(T_2)$ respectively. By the induction hypothesis, that is the case if they satisfy the constraint similar to \eqref{discrete differentiation} but for $T_1$ and $T_2$. In the latter two presentations, if one of them is identically zero, the other may be chosen to be zero as well and of course the constant function zero can be implemented on any tree. So suppose in \eqref{product form 1} and \eqref{product form 2} each of $P_1$ and $P_2$ is non-zero for certain inputs $\tilde{\mathbf{x}}_0\in\{0,1\}^s$ or $\tilde{\tilde{\mathbf{x}}}_0\in\{0,1\}^{n-s}$. Pick $\tilde{\mathbf{x}}_0,\tilde{\tilde{\mathbf{x}}}_0$ arbitrarily in the case of \eqref{sum form}. Now evaluating  at such points gives expressions of $P_1$ and $P_2$ as
\begin{equation*}
\begin{split}
&P_1(x_1,\dots,x_s)= P(x_1,\dots,x_s;\tilde{\tilde{\mathbf{x}}}_0)+\alpha\\
&P_2(x_{s+1},\dots,x_n)= P(\tilde{\mathbf{x}}_0;x_{s+1},\dots,x_n)+\beta;   
\end{split}
\end{equation*}
for some appropriate $\alpha,\beta\in\{0,1\}$. Thus $P_1$ and $P_2$ must satisfy \eqref{discrete differentiation} once the leaves $x_i,x_j,x_l$ there are all from the left sub-tree $T_1$ or from the right sub-tree $T_2$. 
\end{proof}

\subsection{A metric on the set of labeled binary trees}\label{distance}
The spaces of discrete functions associated with binary trees in \S\ref{discrete-tree} could be employed to quantify how different two such trees are. It must be mentioned that the difference is not merely caused by combinatorially distinct underlying graphs, but different labeling schemes may also give rise to different function spaces; keep in mind that in our context binary trees always come with leaves labeled by coordinate functions. Hence we are dealing with the following set:

\begin{definition}\label{tree set}
By ${\rm{Tree}}_n$ we denote the set of all binary rooted trees with $n$ terminals labeled by $x_1,\dots,x_n$ modulo isomorphism of rooted trees. 
\end{definition}
Modding out by isomorphisms is because operations such as swapping labels of two sibling leaves leave the function space unchanged. By abuse of notation, we denote both a binary rooted tree and its class in ${\rm{Tree}}_n$ by the same symbol $T$. Figure \ref{fig:metric''} illustrates elements of ${\rm{Tree}}_4$ where labels come from the set $\{x,y,z,w\}$ rather than $\{x_1,x_2,x_3,x_4\}$.  

Symmetric difference of discrete functions spaces furnishes ${\rm{Tree}}_n$ with a natural metric:

\begin{definition}
The distance between two trees, $T_1$ and $T_2$, is defined by the symmetric difference metric:
\begin{equation}
\begin{split}\label{tree metric}
{\rm{d}}(T_1,T_2)=&\frac{1}{\big|\mathcal{F}_{\rm{bin}}(T_1)\big|+\big|\mathcal{F}_{\rm{bin}}(T_2)\big|}.\big|\mathcal{F}_{\rm{bin}}(T_1)\Delta\mathcal{F}_{\rm{bin}}(T_2)\big|\\
&=\frac{1}{\big|\mathcal{F}_{\rm{bin}}(T_1)\big|}.\big|\mathcal{F}_{\rm{bin}}(T_1)-\mathcal{F}_{\rm{bin}}(T_2)\big|\\
&=\frac{1}{\big|\mathcal{F}_{\rm{bin}}(T_2)\big|}.\big|\mathcal{F}_{\rm{bin}}(T_2)-\mathcal{F}_{\rm{bin}}(T_1)\big|.
\end{split}
\end{equation}
\end{definition}
Keep in mind that the sizes of $\mathcal{F}(T_1)$ and $\mathcal{F}(T_2)$ are the same and dependent only on $n$ by the virtue of Corollary \ref{formula}.

Each constraint imposed in \eqref{temp2} or \eqref{discrete differentiation} on elements of $\mathcal{F}(T)$ or $\mathcal{F}_{\rm{bin}}(T)$ reveals something about the combinatorics of the tree $T$: the most immediate common ancestor of leaves $x_i,x_j$ is not an ancestor of the leaf $x_l$. The proposition below shows that one can indeed reconstruct $T$ from the knowledge of function spaces $\mathcal{F}_{\rm{bin}}(T)$ or 
$\mathcal{F}(T)$. This result is interesting in its own sake as Corollary \ref{formula} says that the cardinality of $\mathcal{F}_{\rm{bin}}(T)$ is only dependent on the number of leaves of $T$; or in the analytic setting, in the sense of Proposition \ref{binomial growth} the number of algebraically independent constraints has nothing to do with the morphology of the tree $T$ with $n$ leaves. 

\begin{proposition}\label{distinguish}
A rooted binary $T$ with $n$ terminals could be recovered from the corresponding set $\mathcal{F}_{\rm{bin}}(T)$ of bit-valued functions $\{0,1\}^n\rightarrow\{0,1\}$. The same is true for the function space 
$\mathcal{F}(T)$.
\end{proposition}

\begin{proof}
As usual, we label the terminals of $T$ by $x_1,\dots,x_n$ so that elements of $\mathcal{F}_{\rm{bin}}(T)$ can be thought of as polynomials in $\Bbb{Z}_2\left[x_1,\dots,x_n\right]$ whose degrees w.r.t. every indeterminate is at most one. Given any three different leaves $x_{i_s}\, (s\in\{1,2,3\})$, there is one of them which is farther apart from the other two in the sense that it lies outside a rooted sub-tree which contains the other two. The knowledge of the outsider leaf in every possible triple of leaves completely determines the combinatorics of the tree. The constraints \eqref{discrete differentiation} indicate that
$$
\frac{\partial^2 P}{\partial x_{i_1}\partial x_{i_3}}.\frac{\partial P}{\partial x_{i_2}}=\frac{\partial^2 P}{\partial x_{i_2}\partial x_{i_3}}.\frac{\partial P}{\partial x_{i_1}}
$$
if $x_{i_3}$ is the one which is separated from $x_{i_1},x_{i_2}$. Consequently, it suffices to come up with a polynomial that satisfies the constraint above but not the analogous ones 
$$
\frac{\partial^2 P}{\partial x_{i_2}\partial x_{i_1}}.\frac{\partial P}{\partial x_{i_3}}=\frac{\partial^2 P}{\partial x_{i_3}\partial x_{i_1}}.\frac{\partial P}{\partial x_{i_2}},\quad
\frac{\partial^2 P}{\partial x_{i_1}\partial x_{i_2}}.\frac{\partial P}{\partial x_{i_3}}=\frac{\partial^2 P}{\partial x_{i_3}\partial x_{i_2}}.\frac{\partial P}{\partial x_{i_1}}
$$
corresponding to situations where $x_{i_1}$ or $x_{i_2}$ is the outsider of $\{x_{i_1},x_{i_2},x_{i_3}\}$.
The polynomial $\left(x_{i_1}+x_{i_2}\right)x_{i_3}$ has all of these properties. Moreover, it obviously belongs to $\mathcal{F}_{\rm{bin}}(T)$ if $x_{i_3}$ is separated from $x_{i_1},x_{i_2}$ via a rooted sub-tree: the node which is the root of this sub-tree gets $x_{i_{1}}$ and $x_{i_{2}}$ as inputs and adds them. The resulting sum is then passed to the root of $T$ along with $x_{i_3}$, and then a multiplication takes places at the root.\\
\indent
The proof above works equally well in the analytic setting as one can think of $\left(x_{i_1}+x_{i_2}\right)x_{i_3}$ as an analytic function $\Bbb{R}^n\rightarrow\Bbb{R}$ rather than a binary one $\{0,1\}^n\rightarrow\{0,1\}$. 
\end{proof}

\begin{corollary}
The function 
${\rm{d}}:{\rm{Tree}}_n\times{\rm{Tree}}_n\rightarrow [0,1]$ from
\eqref{tree metric} defines a metric on the set ${\rm{Tree}}_n$ of labeled binary rooted trees.
\end{corollary}

\begin{proof}
Symmetry and triangle inequality trivially hold. Proposition \ref{distinguish} implies that ${\rm{d}}$ is a metric rather than a pseudo-metric as ${\rm{d}}(T_1,T_2)=0$ implies $T_1=T_2$.
\end{proof}
\begin{figure}
    \centering
    \includegraphics[width=14cm]{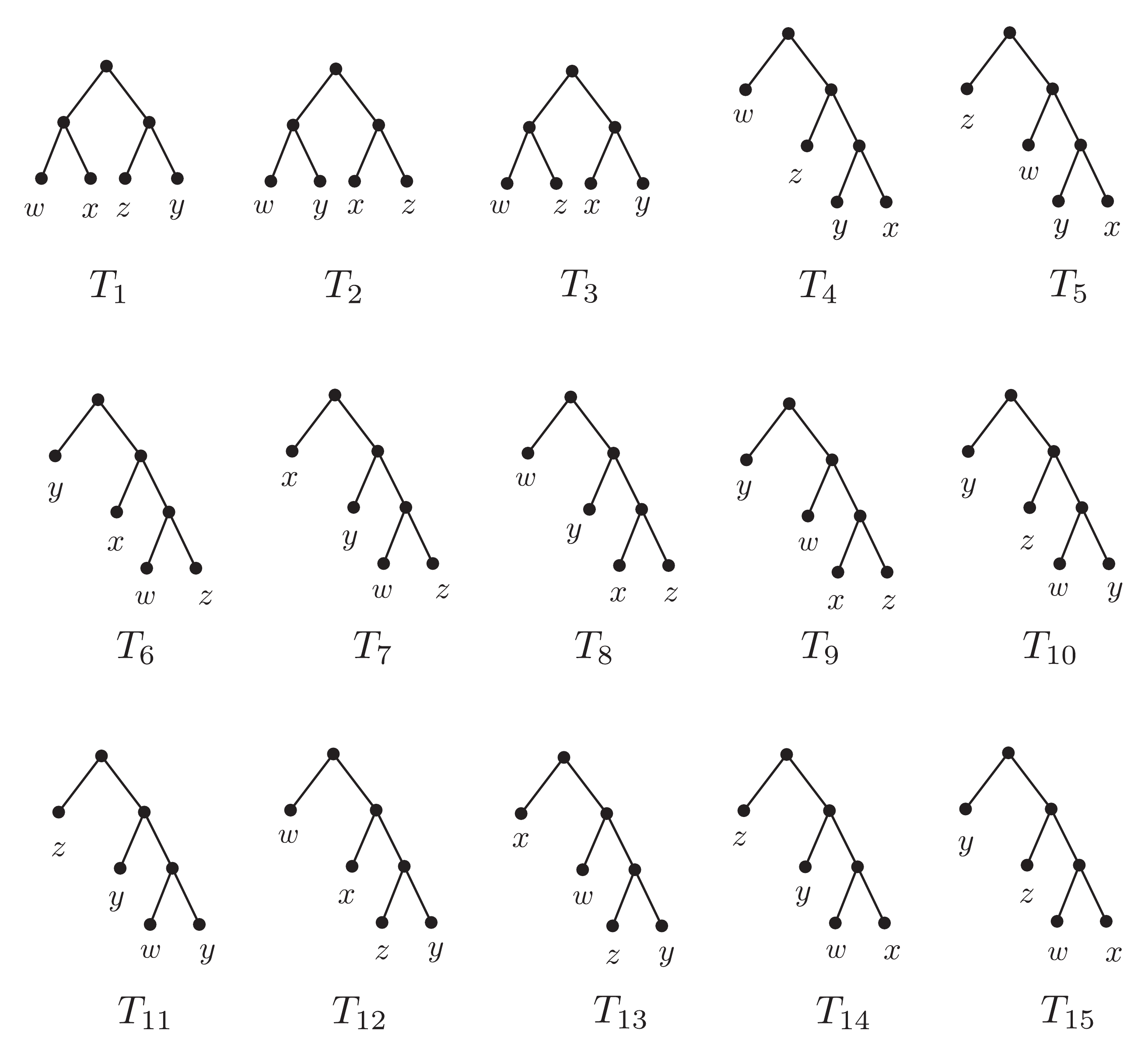}
    \caption{All elements of the set ${\rm{Tree}}_4$ of labeled binary trees with four terminals.}
    \label{fig:metric''}
\end{figure}

\begin{example}\label{4var-2}
This example, like Example \ref{4var-1}, investigates binary trees with four terminals. Here, we calculate the distances between some of the binary trees illustrated in Figure \ref{fig:metric''}. By the virtue of \eqref{binary function formula}, the total size of each function space is $520$ when $n=4$. Let us first consider two different trees e.g. those of Figure \ref{fig:three-four} which appear as $T_3,T_4$ in Figure \ref{fig:metric''}. We need to compute how many four variable  functions $F(x,y,z,w)$ admitting a representation by the symmetric tree $T_3$ can also be represented by the asymmetric one $T_4$. Based on the results of \S\ref{discrete-tree} and Theorem \ref{main binary}, $F(x,y,z,w)$ could be thought of as a polynomial whose degree with respect to each indeterminate is one (i.e. is in the form \eqref{polynomial space}) that moreover, satisfies \eqref{four-1}. It has a representation by the latter tree if and only if it satisfies the constraints \eqref{four-2}. Comparing these two groups of equations, one only requires the identities below to hold in the ring $\Bbb{Z}_2[x,y,z,w]$:
\begin{equation}\label{auxiliary 12}
\frac{\partial^2 F}{\partial z\partial w}\frac{\partial F}{\partial x}=\frac{\partial^2 F}{\partial x\partial w}\frac{\partial F}{\partial z};\quad
\frac{\partial^2 F}{\partial z\partial w}\frac{\partial F}{\partial y}=\frac{\partial^2 F}{\partial y\partial w}\frac{\partial F}{\partial z}.
\end{equation}
The description \eqref{description} of tree functions in terms of sub-trees comes in handy now. Constant functions and bivariate functions of $x,y$ or of $z,w$ could definitely be represented by the asymmetric tree. Next, we determine which functions of the form $F_1(x,y)F_2(z,w)$, $F_1(x,y)F_2(z,w)+1$ or $F_1(x,y)+F_2(z,w)$ with $F_1,F_2$ being non-constant polynomials of the form \eqref{polynomial space} satisfy \eqref{auxiliary 12}. Plugging the first two in 
\eqref{auxiliary 12} yields
$$
F_1\frac{\partial F_1}{\partial x}\left[\frac{\partial^2 F_2}{\partial z\partial w}F_2-\frac{\partial F_2}{\partial z}\frac{\partial F_2}{\partial w}\right]=0;\quad
F_1\frac{\partial F_1}{\partial y}\left[\frac{\partial^2 F_2}{\partial z\partial w}F_2-\frac{\partial F_2}{\partial z}\frac{\partial F_2}{\partial w}\right]=0.
$$
Either $\frac{\partial F_1}{\partial x}$ or $\frac{\partial F_1}{\partial y}$ is non-zero. Thus \eqref{auxiliary 12} boils down to 
$\frac{\partial^2 F_2}{\partial z\partial w}F_2=\frac{\partial F_2}{\partial z}\frac{\partial F_2}{\partial w}$; and this holds if and only if $F_2(z,w)$ is a product of linear polynomials of $z,w$; that is, $F_2(z,w)$ is one of the followings:
$$z,z+1,w,w+1,zw,(z+1)(w+1),z(w+1),(z+1)w.$$
Next, writing \eqref{auxiliary 12} for $F=F_1(x,y)+F_2(z,w)$ (where $F_1(0,0)=0$ as in \eqref{description}) yields
$$
\frac{\partial^2 F_2}{\partial z\partial w}\frac{\partial F_1}{\partial x}=\frac{\partial^2 F_2}{\partial z\partial w}\frac{\partial F_1}{\partial y}=0.
$$
Since $F_1(x,y)$ is non-constant, this means $\frac{\partial^2 F_2}{\partial z\partial w}=0$, i.e. $F_2$ is one of the affine polynomials below:
$$
z,z+1,w,w+1,z+w,z+w+1.
$$
Now for each subset appearing in the disjoint union \eqref{description} we know how many functions admit representations by both symmetric and asymmetric trees in Figure \ref{fig:three-four}. This enables us to calculate the size of the intersection of the corresponding function spaces as
$$
(16-2)\times 8+(16-2)\times 8+\frac{16-2}{2}\times 6+(16-2)+(16-2)+2=296.
$$
Hence the desired distance is $\frac{1}{520}(520-296)\approx 0.43$.\\
\indent 
We conclude the example by finding the distance for a case where a tree is considered with two different labels. Take the symmetric trees $T_2$ and $T_3$ in the top row of Figure \ref{fig:metric''}. Invoking the description \eqref{description} of tree functions $F(x,y,z,w)$ in terms of sub-trees again, it is not hard to see that the only tree functions in common are those in the product form
\begin{equation}\label{auxiliary 13}
F(x,y,z,w)=G_1(x)G_2(y)G_3(z)G_4(w)
\end{equation}
or in the sum form
\begin{equation}\label{auxiliary 14}
F(x,y,z,w)=G_1(x)+G_2(y)+G_3(z)+G_4(w).
\end{equation} 
This is due to the fact that a polynomial identity such as $F_1(x,y)F_2(z,w)\equiv \tilde{F}_1(x,z)\tilde{F}_2(y,w)$ (resp. $F_1(x,y)+F_2(z,w)\equiv \tilde{F}_1(x,z)+\tilde{F}_2(y,w)$) requires all constituent parts (assumed to be non-constant) to be product (resp. sum) of single variable polynomials. Hence the cardinality of the intersection is
$$
2^4+\binom{4}{1}\times 2^3+\binom{4}{2}\times 2^2+2^5=104;
$$
where the summands respectively account for the following possibilities for $F$:
\begin{itemize}
\item all four polynomials in  \eqref{auxiliary 13} are of degree one;
\item three of polynomials in \eqref{auxiliary 13} are of degree one and the other is constant $1$;
\item two of polynomials in  \eqref{auxiliary 13} are of degree one and the other two are constant $1$;
\item $F$ is a linear combination of $1,x,y,z,w$ over $\Bbb{Z}_2$ like \eqref{auxiliary 14}.
\end{itemize}
Consequently, we arrive at
$${\rm{d}}(T_2,T_3)=\frac{1}{520}(520-104)\approx 0.80.$$
\end{example}

\begin{figure}
    \centering
    \includegraphics[width=12cm]{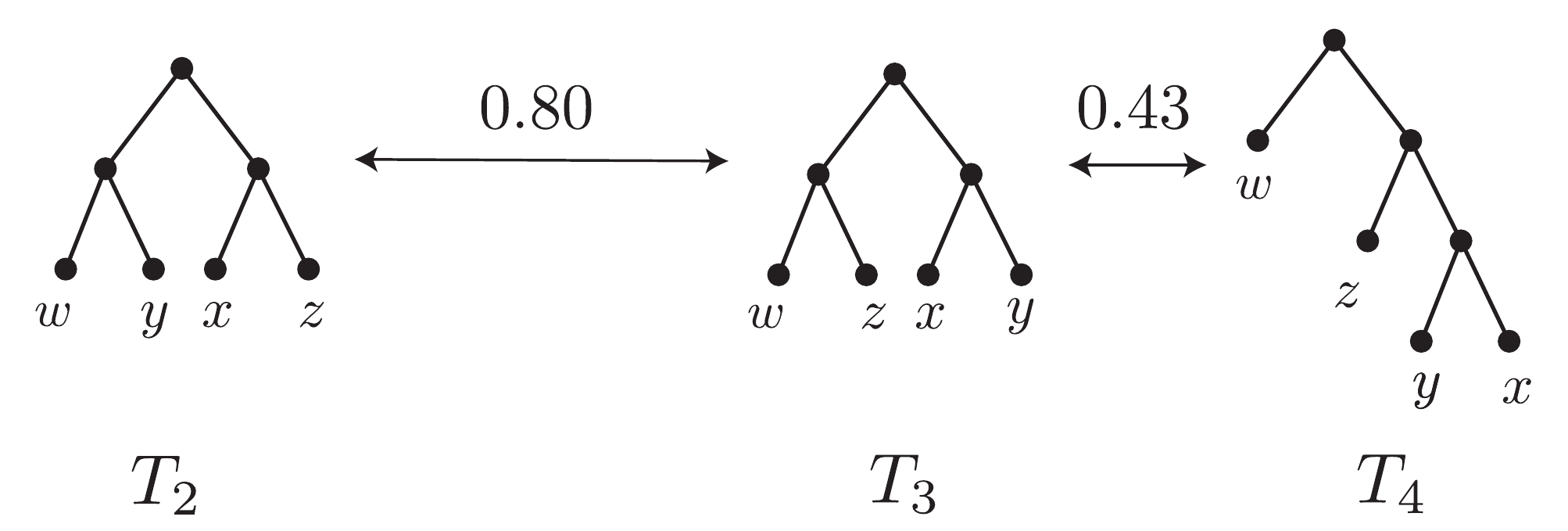}
    \caption{The distances between two pairs of the trees from Figure \ref{fig:metric''} as calculated in Example \ref{4var-2}.}
    \label{fig:metric'''}
\end{figure}

\section{Application to neural networks}\label{neural}
The goal of this section is to adapt the techniques used so far to study feed-forward neural networks. In the first subsection, we present a method for passing from such a neural network to a rooted tree. In the later subsections we try to imitate arguments of \S\ref{continuous setting} and \S\ref{discrete setting} for more general trees. Two main difficulties come up then:
\begin{itemize}
\item the number of leaves is probably larger than the number of variables of the function to be implemented; in other words, different leaves could correspond to the same variable;
\item trees are not binary anymore.
\end{itemize}
In this section trees are, as always, rooted with their leaves labeled. One can adapt the terminology of \S\ref{tree functions} to non-binary trees with no difficulty. The number of leaves/terminals is again denoted by  $n$. Other vertices, called nodes/branch points/non-terminals, have at least two and at most $c$ successors where $c$ is the maximum number of children/successors of a node.\footnote{For the sake of inductive argument, we consider a single vertex to be a tree with only one leaf.}  

\subsection{From neural networks to trees}\label{neural to tree}
Trees are just one particular architecture  for neural networks. Nevertheless, they could serve as building blocks. A neural network is made up of multiple layers stacking atop of each other with each layer containing several nodes. Denoting the layers by $L_1$ up to $L_n$, the nodes in the first layer $L_1$ correspond to the inputs. Associated with each node of the $i^{\rm{th}}$ layer is a function which as its inputs takes the outputs of a subset of nodes (at least two) in the $i-1^{\rm{th}}$ layer.   The first layer takes the input and after that the functions associated to the nodes  in the next layer compute the values for this layer. This continues up to the last layer where the output of the neural network is returned at the root. Therefore, once again superpositions of functions take place. Unless otherwise stated, we assume functions applied at nodes to be either analytic or bit-valued.

Tree functions are examples of functions computed via neural networks. Conversely, a feed-forward neural network $N$ may be converted to its corresponding \textbf{TENN} (the \textbf{T}ree \textbf{E}xpansion of the \textbf{N}eural \textbf{N}etwork) which is a (not necessarily binary) rooted tree $T$. This construction is best demonstrated in Figure \ref{fig:tree-expanded-network}. Abstractly speaking, each vertex of the tree $T$ represents a path starting from a vertex of the neural network $N$ and ascending through the layers until it reaches the root of $N$. Given such a path, for any two sub-paths to the root determined by successive vertices of $N$ the corresponding vertices of $T$ must be connected. \\
\indent 
We should keep in mind that in this expansion  nodes and leaves might be repeated (both for the input layer and upper layers). Consequently, unlike the convention in \S\ref{tree functions}, we are dealing with trees for which different leaves may share a label.\\

\subsection{Trees with repeated labels; a toy example}\label{labeled; exact}
In Theorem \ref{main} we characterized analytic tree functions via a system of PDEs. The working assumption in that theorem and also throughout \S\ref{continuous setting}, \S\ref{discrete setting} 
was that the leaves of the tree correspond to different variables. On the contrary, as discussed in the preceding section, the tree resulting from expanding a neural network is not necessarily binary and there are probably different leaves labeled by the same variable. \\
\indent
To demonstrate how cumbersome a PDE constraint may get in presence of repetition, let us go back to the toy example of a function of $x,y,z$ from \S\ref{three-variable-example}, but this time in the form of  
\begin{equation}\label{3var_form_alternate}
F(x,y,z) = g(f(x,y),h(x,z))
\end{equation}
rather than \eqref{3var_form}.
This is a superposition of  function $f(x,y), h(x,z)$ and $g(u,v)$ and  is illustrated in Figure \ref{fig:tree-expanded-network-2}. Switching to the subscript notation for partial derivatives, we try the imitate the usage of chain rule in \S\ref{three-variable-example} which resulted in constraint \eqref{condition}. 
Differentiating yields: 
\begin{equation*}
F_y = g_u.f_y,\quad F_z = g_v.h_z,\quad F_x = g_u.f_x+g_v.h_x;
\end{equation*}
and therefore, the linear combination
\begin{equation}\label{first-derivation-nn}
F_x = F_y\big(\frac{f_x}{f_y}\big)+F_z\big(\frac{h_x}{h_z}\big)=A(x,y).F_y+B(x,z).F_z.
\end{equation}
Notice that weights $A,B$ in the linear combination above do not depend on $z$ and $y$. Taking higher partial derivatives of \eqref{first-derivation-nn} with respect to $y,z$ results in a description of derivatives of the form $F_{xy\dots yz\dots z}$ as linear combinations of functions such as $F_{y\dots y}$ and $F_{z\dots z}$. But assuming that we have differentiated up to $n$ times, there are ${n+2 \choose 2}$ partial derivatives such as  $F_{xy\dots yz\dots z}$ in which the total order of differentiation with respect to $y,z$ is at most $n$. On the other side of the linear combination, the weights could only be partial derivatives $A_{y\dots y}$ and $B_{z\dots z}$ of order not greater than $n$. But there are $2(n+1)$ such terms. This implies that there is a linear dependency of the aforementioned partial derivatives of $F$ provided that $n$ is large enough. And a linear dependency could always be described as the vanishing of determinants.  \\
\indent
For example, by differentiating  \eqref{first-derivation-nn} up to three times, we arrive at:
\small
\begin{equation}\label{auxiliary 15}
\begin{bmatrix}
F_x \\F_{xy}\\F_{xz}\\F_{xyz}\\F_{xyy}\\F_{xzz}\\F_{xyzz}
\end{bmatrix}
=
\begin{bmatrix}
F_y&F_z&0&0&0&0\\
F_{yy}&F_{yz}&F_y&0&0&0\\
F_{yz}&F_{zz}&0&F_z&0&0\\
F_{yyz}&F_{yzz}&F_{yz}&F_{yz}&0&0\\
F_{yyy}&F_{yyz}&2F_{yy}&0&F_y&0\\
F_{yzz}&F_{zzz}&0&2F_{zz}&0&F_z\\
F_{yyzz}&F_{yzzz}&F_{yzz}&2F_{yzz}&0&F_{yz}\\
\end{bmatrix}
\begin{bmatrix}
A\\B\\A_y\\B_z\\A_{yy}\\B_{zz}.
\end{bmatrix}
\end{equation}
\normalsize
\begin{proof}[Proof of Proposition \ref{pde-for-g(f(x,y),h(x,z))}]
The vector on the left of \eqref{auxiliary 15} is in the column space of the $7\times 6$ matrix appeared on the right. Joining them results in a $7\times 7$ matrix whose columns are linearly dependent and its determinant thus must be zero.   
\end{proof}
Article \cite{arnold2009representation1} alludes to the function $xy+yz+zx$ as a function that is not a superposition of the form \eqref{3var_form_alternate} of continuous bivariate functions. Notice that this function satisfies  the condition of Proposition \ref{pde-for-g(f(x,y),h(x,z))}; hence that condition is not sufficient. Example below addresses a similar issue with our recurring non-example.

\begin{figure}
    \centering
    \includegraphics[width=9cm]{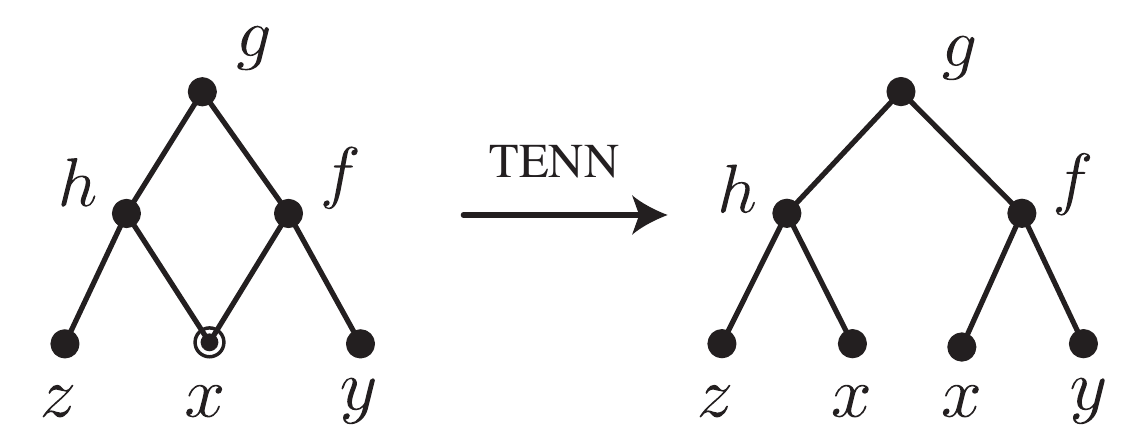}
    \caption{The superposition $F(x,y,z) = g(f(x,y),h(x,z))$ as a neural network with three inputs (left) and the expansion of this network to a tree with four leaves (right).}
    \label{fig:tree-expanded-network-2}
\end{figure}

\begin{example}
The familiar function $xyz+x+y+z$ from \eqref{non-example} failed to satisfy the constraint \eqref{condition} while it is not hard to see that it is a solution to the PDE from Proposition \ref{pde-for-g(f(x,y),h(x,z))}. Here we are going to show that it cannot be represented globally as a superposition of the form  \eqref{3var_form_alternate} (and in retrospect as a superposition of the form \eqref{3var_form}) even in the continuous category. Aiming for a contradiction, suppose    
\begin{equation}\label{auxiliary 17}
g(f(x,y),h(x,z))=xyz+x+y+z
\end{equation}
for globally defined continuous functions $f,g,h:\Bbb{R}^2\rightarrow\Bbb{R}$. 
Substituting $y$ with $\frac{-1}{x}$ yields
\begin{equation}\label{auxiliary 18}
g\left(f\left(x,\frac{-1}{x}\right),h(x,z)\right)=x-\frac{1}{x}
\end{equation}
for any $x,y,z$ with $x\neq 0$. Let $x_0\neq 0$ and $z_0$ be arbitrary. The function $z\mapsto h(x_0,z)$ could not be constant over a non-degenerate interval because then 
$$
(y,z)\mapsto g(f(x_0,y),h(x_0,z))=x_0yz+x_0+y+z
$$
would be independent of $z$ on some non-degenerate rectangular region of the plane. 
Hence there is an $\epsilon>0$ such that the image of $z\mapsto h(x_0,z)$ contains an open interval of the form 
$(a-2\epsilon,a+2\epsilon)$ where $a:=h(x_0,z_0)$. By continuity, for sufficiently small $\delta>0$ the image of $z\mapsto h(x_1,z)$ contains the interval $J:=(a-\epsilon,a+\epsilon)$ provided that $|x_0-x_1|<\delta$. Now for any such an $x_1$ the function $g$ is constant on the vertical segment 
$$
\left\{f\left(x_1,\frac{-1}{x_1}\right)\right\}\times J
$$
because of \eqref{auxiliary 18}. Hence if 
\begin{equation}\label{auxiliary 19}
x\mapsto f\left(x,\frac{-1}{x}\right)
\end{equation}
is not constant in vicinity of $x_0$, then $g(u,v)$ would be dependent only on its first variable for $(u,v)$ from a non-degenerate rectangle containing $\left(f\left(x_0,\frac{-1}{x_0}\right),a\right)=\left(f\left(x_0,\frac{-1}{x_0}\right),h(x_0,z_0)\right)$. But for $(x,y,z)$ close enough to $\left(x_0,\frac{-1}{x_0},z_0\right)$ the point 
$$
(f(x,y),h(x,z))
$$
and hence $g(f(x,y),h(x,z))$ is dependent only on $x,y$ as $(x,y,z)$ varies in a sufficiently small neighborhood of $\left(x_0,\frac{-1}{x_0},z_0\right)$. This is of course impossible since  the right hand side of \eqref{auxiliary 17} is not independent of $z$ when $y\neq\frac{-1}{x}$.  Therefore, \eqref{auxiliary 19} must be constant over an interval $(x_0-\delta',x_0+\delta')$ where $0<\delta'\leq\delta$. We next arrive at a contradiction: Pick an $x_1\neq x_0$ form this interval. By the previous discussion, there is a number $z_1$ such that 
$$
h(x_1,z_1)=h(x_0,z_0)=a.
$$
But then 
$$
\left(f\left(x_0,\frac{-1}{x_0}\right),h(x_0,z_0)\right)=\left(f\left(x_1,\frac{-1}{x_1}\right),h(x_1,z_1)\right)
$$
which by \eqref{auxiliary 18} requires $x_0-\frac{1}{x_0}$ and $x_1-\frac{1}{x_1}$ to be the same; a contradiction as $x\mapsto x-\frac{1}{x}$ is injective.  
\end{example}

The preceding discussion indicates that in the case of trees with repeated labels (hence, following the discussion in \S\ref{neural to tree}, the case of neural networks) numerous PDE constraints originated from the linear algebra could be imposed on tree functions. In Theorem \ref{main} we observed that in the case of distinct labels, a finite number of PDEs are enough to pin down analytic tree functions. This promotes us to ask for a generalization:
\begin{question}\label{conjecture}
Is it true that analytic superpositions corresponding to any given architecture of neural networks may be characterized as solutions to a finite system of PDEs?
\end{question}

\subsection{Trees with repeated labels; estimates}\label{labeled; estimate}
After realizing the functions produced by neural networks as tree functions for trees with repeated labels, we next proceed to invoke our results to study tree function spaces in presence of repeated labels. In the context of bit-valued functions, obtaining a precise count of discrete functions such as \eqref{binary function formula} is implausible. Nevertheless, one could easily bound the number of such functions. Working with the notation fixed in the beginning of this section and fixing a tree $T$ with $n$ leaves, the tree functions under consideration are of the form
$$
F:\{0,1\}^p\rightarrow\{0,1\}\quad (x_1,\dots,x_p)\mapsto F(x_1,\dots,x_p)
$$
where $x_1,\dots,x_p$ are the labels of the leaves of $T$; so $p\leq n$. These functions are superpositions of functions of the form $\{0,1\}^i\rightarrow\{0,1\}$ each assigned to a node where $2\leq i\leq c$ is the  number of inputs that the node receives. Hence a very crude overestimate for $\left|\mathcal{F}_{\rm{bin}}(T)\right|$ is given by 
\begin{equation}\label{crude}
\left(2^{2^c}\right)^{\#\text{ of nodes of }T}.
\end{equation}
The number of nodes/non-terminal vertices of $T$ could be easily bounded in terms of $n$:

\begin{lemma}\label{lemma 1}
A rooted tree $T$ with $n$ leaves has at most $n-1$ nodes with equality if and only if $T$ is binary.
\end{lemma}
\begin{proof}
From the basic graph theory (\cite[chap. 1]{MR2440898}) the number of edges of the underlying graph of $T$ is simultaneously the number of vertices minus one and half the sum of degrees of all vertices. The former is 
$$
\#\text{ of nodes }+\#\text{ of leaves }-1=\#\text{ of nodes }+n-1;
$$
while the latter is at least
$$
\frac{1}{2}\left(3\left(\#\text{ of nodes }-1\right)+2+n\right)
$$
due to the fact that the degree of every node other than the root itself is at least three (it has one predecessor and at least two successors). Therefore:
$$
\#\text{ of nodes }+n-1\geq \frac{1}{2}\left(3\left(\#\text{ of nodes }-1\right)+2+n\right)\Leftrightarrow n-1\geq \#\text{ of nodes }.
$$
The equality occurs exactly when every non-root node has only two successors; i.e. the rooted tree $T$ is binary. 
\end{proof}
This lemma yields $\left(2^{2^c}\right)^{n-1}$ as an upper bound for the size of the discrete TFS $\mathcal{F}_{\rm{bin}}(T)$ of the tree $T$ with $n$ leaves appeared before. Invoking the ideas developed in the proof of Theorem \ref{discrete}, this bound could be improved: Assuming that the root of $T$ is of degree $m$ ($2\leq m\leq c$), as its inputs it receives the outputs of $m$ tree functions $F_1,\dots,F_m$ implemented on the smaller sub-trees emanating from the root. These outputs are then passed to a function $g:\{0,1\}^m\rightarrow\{0,1\}$ at the root and at last, the procedure results in $f:=g\left(F_1,\dots,F_m\right)$ as our tree function. The key point is that it is not necessary to consider all $2^m$ possibilities for $g$ since if, for instance, one modifies $g(x_1,x_2,\dots,x_m)$ as $g(x_1+1,x_2,\dots,x_m)$, then feeding the root with the bit-valued function $F_1-1=F_1+1$ instead of $F_1$ results in the same function $f$; notice that $F_1+1$ is obviously a tree function for the same sub-tree. In other words, identifying the set of functions $\left\{0,1\right\}^m\rightarrow\left\{0,1\right\}$ with the set  
$$
\left\{\sum_{S\subseteq \{1,\dots,m\}}\epsilon(S)\prod_{s\in S}x_s\,\big|\epsilon(S)\in\{0,1\}\right\},
$$
of binary polynomials (just like what we did in \eqref{polynomial space}), we are going to argue that the set above partitions to certain equivalence classes such that the choices for the function assigned to the root could be narrowed down to a set of representatives of these classes; hence a fewer number of choices. This is a direct generalization of what we did for binary trees in the proof of Theorem \ref{discrete} where we divided functions $\{0,1\}^2\rightarrow\{0,1\}$ into seven groups. More precisely, two binary polynomials $g(x_1,\dots,x_m)$ and  $\tilde{g}(x_1,\dots,x_m)$ of the form above (meaning of degree at most one with respect to each indeterminate) are regarded to be equivalent if one is obtained from the other by changing $x_i$ to $x_i+1$ for a subset of indeterminates. We shall need the number of such equivalence classes. This is the content of the lemma below.
\begin{lemma}\label{lemma 2}
The number of classes of the equivalence relation on the set 
\begin{equation}\label{auxiliary 16}
\left\{\sum_{S\subseteq \{1,\dots,m\}}\epsilon(S)\prod_{s\in S}x_s\,\big|\epsilon(S)\in\{0,1\}\right\},
\end{equation}
defined above is $2^{2^{m-1}-m}\left(2^{2^{m-1}}+2^m-1\right)$.
\end{lemma}

\begin{proof}
The proof is an application of Burnside's lemma \cite[\S 2.7]{MR2440898}. These equivalence classes are orbits of an action of the group $(\Bbb{Z}_2)^m$ on the subset \eqref{auxiliary 16} of $\Bbb{Z}_2\left[x_1,\dots,x_m\right]$ defined as
$$
\left(\delta_1,\dots,\delta_m\right).g(x_1,\dots,x_m):=g\left(x_1+\delta_1,\dots,x_m+\delta_m\right)
$$
for any $m$-tuple $\boldsymbol{\delta}:=\left(\delta_1,\dots,\delta_m\right)$ of $0$'s and $1$'s. Burnside's lemma gives the number of orbits as the sum of the number of  polynomials in \eqref{auxiliary 16} fixed by $\boldsymbol{\delta}$  as $\boldsymbol{\delta}$ varies in $(\Bbb{Z}_2)^m$ divided  by the cardinality $2^m$ of $(\Bbb{Z}_2)^m$. It is not hard to see that an non-identity element $\boldsymbol{\delta}$ fixes exactly half of the polynomials; e.g. $\boldsymbol{\delta}=(1,0,\dots,0)\in(\Bbb{Z}_2)^m$ fixes only those members of \eqref{auxiliary 16} in which $x_1$ does not appear and there are $2^{2^{m-1}}$ of them. The general situation for a group element
$$
\boldsymbol{\delta}=\left(\delta_1,\dots,\delta_m\right)
$$
in which one $\delta_i$, say $\delta_1$, is non-zero could be reduced to the aforementioned case by the linear change of coordinate 
$$
(x_1,x_2,\dots,x_m)\mapsto\left(\tilde{x}_1,\tilde{x}_2,\dots,\tilde{x}_m\right)=\left(x_1,x_2+\delta_2x_1,\dots,x_m+\delta_mx_1\right)  
$$
due to the fact that $\boldsymbol{\delta}$ takes $\tilde{x}_1$ to $\tilde{x}_1+1$ while $\tilde{x}_2,\dots,\tilde{x}_m$ are preserved. We conclude that the number of orbits is:
$$
\frac{1}{2^m}\left((2^m-1)\times 2^{2^{m-1}}+2^{2^{m}}\right)=2^{2^{m-1}-m}\left(2^{2^{m-1}}+2^m-1\right).
$$
\end{proof}
Now we arrive at the following improvement of the bound $\left(2^{2^c}\right)^{n-1}$ derived before:
\begin{proposition}\label{estimate}
Let $T$ be a tree with $n$ leaves whose nodes have at most $c\geq 2$ children. Then 
$$
\left|\mathcal{F}_{\rm{bin}}(T)\right|\leq 4^n\times\left(2^{2^{c-1}-c}\left(2^{2^{c-1}}+2^c-1\right)\right)^{n-1}. 
$$
\end{proposition}

\begin{proof}
In the base case, when $n=1$, there are exactly four functions $\{0,1\}\rightarrow\{0,1\}$ and the inequality is clearly satisfied. For the inductive step and with notation as before, removing the root of $T$ leaves us with $m$ sub-trees $T_1,\dots,T_m$ where $2\leq m\leq c$. Denoting numbers of their leaves by $n_1,\dots,n_m$, one has $n_1+\dots+n_m=n$ and by the induction hypothesis
$$
\left|\mathcal{F}_{\rm{bin}}(T_i)\right|\leq 4^{n_i}\times\left(2^{2^{c-1}-c}\left(2^{2^{c-1}}+2^c-1\right)\right)^{n_i-1}
$$
for all $1\leq i\leq m$. Lemma \ref{lemma 2} now implies that 
\begin{equation*}
\begin{split}
\left|\mathcal{F}_{\rm{bin}}(T)\right|
&\leq\left(2^{2^{m-1}-m}\left(2^{2^{m-1}}+2^m-1\right)\right)\prod_{i=1}^m\left|\mathcal{F}_{\rm{bin}}(T_i)\right|\\
&\leq \left(2^{2^{c-1}-c}\left(2^{2^{c-1}}+2^c-1\right)\right)\prod_{i=1}^m4^{n_i}\times\left(2^{2^{c-1}-c}\left(2^{2^{c-1}}+2^c-1\right)\right)^{n_i-1}\\
&\leq 4^n\left(2^{2^{c-1}-c}\left(2^{2^{c-1}}+2^c-1\right)\right)^{n-m+1}\leq 4^n\times\left(2^{2^{c-1}-c}\left(2^{2^{c-1}}+2^c-1\right)\right)^{n-1}.
\end{split}
\end{equation*}
\end{proof}

\begin{corollary}\label{corollary}
A tree $T$ with $n$ leaves whose nodes have no more than $c$ children satisfies $\left|\mathcal{F}_{\rm{bin}}(T)\right|=O\left(\gamma_c^n\right)$ where
\begin{equation}\label{gamma}
\gamma_c:=\begin{cases}
6\hspace{4.6cm} c=2\\
2^{2^{c-1}-c+2}\left(2^{2^{c-1}}+2^c-1\right)\quad c\geq 3
\end{cases}
\end{equation}
is a number smaller than $2^{2^c}$.
\end{corollary}
\begin{proof}
Proposition \ref{estimate} indicates that the number of bit-valued tree functions implemented on $T$ is 
$$
O\left(\left(2^{2^{c-1}-c+2}\left(2^{2^{c-1}}+2^c-1\right)\right)^n\right).
$$
When  $c\geq 3$ the base $2^{2^{c-1}-c+2}\left(2^{2^{c-1}}+2^c-1\right)$  is less than the number  $2^{2^c}$ appeared in the rudimentary bound \eqref{crude}. For $c=2$ the tree is binary and we could use the bound $O\left(6^n\right)$ on $\left|\mathcal{F}_{\rm{bin}}(T)\right|$ that Corollary \ref{formula} provides. 
\end{proof}

We finish with two applications to trees with repeated labels and neural networks. First, suppose the goal is to implement all functions 
$\{0,1\}^p\rightarrow\{0,1\}$ on a rooted tree whose leaves are labeled by $x_1,\dots,x_p$ and let $c$ be the maximum possible number of the children of a node. The essence of the corollary below is that for this goal to be achieved, fixing $c$ the number $n$ of the terminals must be $O(2^p)$ as $p$ grows.

\begin{corollary}\label{bound}
Let $T$ be a tree with $n$ leaves labeled by variables $x_1,\dots,x_p$ whose nodes have at most $c\geq 2$ children and let $\gamma_c$ be as in \eqref{gamma}. Every function $F:\{0,1\}^p\rightarrow\{0,1\}$ could be implemented on this tree only if
$$
n\log(\gamma_c)\geq\log\left(\frac{5\times 2^{2^p}-8}{2}\right)
$$
for $c=2$, and 
$$
(n-1)\log(\gamma_c)\geq 2^p-2
$$
if $c\geq 3$.\footnote{All logarithms are in base $2$.}
\end{corollary}
\begin{proof}
The number $2^{2^p}$ of functions $\{0,1\}^p\rightarrow\{0,1\}$ cannot be greater than $\left|\mathcal{F}_{\rm{bin}}(T)\right|$. When $c=2$, we use formula \eqref{binary function formula} for this cardinality while for $c\geq 3$,  we invoke the upper bound from Proposition \ref{estimate}.
\end{proof}

We finally apply the previous results to tree expansions of neural networks in order to compare different architectures. 

\begin{definition}
For a neural network $N$, the number of leaves in the expanded-tree is denoted by $L(N)$. The corresponding set of bit-valued functions implemented on $N$ is shown by $\mathcal{F}_{\rm{bin}}(N)$.
\end{definition}

\begin{theorem}\label{nn-size-function-space}
For a neural network $N$ in which each node is connected to at most $c$ nodes from the previous layer one has
\begin{equation*}
\left|\mathcal{F}_{\rm{bin}}(N)\right|=O\left(\gamma_c^{L(N)}\right).    
\end{equation*}
\end{theorem}
\begin{proof}
Apply Corollary \ref{corollary} to the TENN that corresponds to $N$.
\end{proof}

\textbf{Acknowledgements.} The authors are grateful to Mehrdad Shahshahani for proposing Remark \ref{Frobenius}, to David Rolnick for reading an early draft of this paper, to Mohammad Ahmadpoor, Ari Benjamin, Ben Lansdell, Pat Lawlor, Samantha Ing-Esteves and Nidhi Seethapathi for helpful conversations, to the organizers of \href{http://www.nasonline.org/programs/sackler-colloquia/completed_colloquia/science-of-deep-learning.html}{\textit{The Science of Deep Learning}} colloquium where a poster based on this work was presented, to the users contributed to answering a  \href{https://mathoverflow.net/questions/322184/continuous-functions-of-three-variables-as-superpositions-of-two-variable-functi}{question} we posted on the MathOverflow network about representations of ternary functions in terms of bivariate functions and to the NIH for funding (R01MH103910).

\bibliography{biblography}
\bibliographystyle{alpha}
\end{document}